\documentclass[pdflatex,sn-mathphys]{sn-jnl}

\usepackage{graphicx,subfig,wrapfig}
\usepackage{lineno,hyperref}
\usepackage{moreverb}

\usepackage{algorithm}
\usepackage{algorithmicx}
\usepackage{algpseudocode}
\usepackage{color}
\usepackage{amsmath}
\usepackage{amsthm}
\usepackage{txfonts}
\usepackage{afterpage}
\usepackage{mathrsfs}
\usepackage{xcolor}

\usepackage{arydshln}
\usepackage{booktabs}
\usepackage[symbol]{footmisc}

\jyear{2021}%

 \font\msbm=msbm10

\def\dsR{\hbox{{\msbm \char "52}}}

\def\lr{\color{black}}

\renewcommand{\mathbf}{\boldsymbol}

\newcommand{\norm}[2]{\left\| #1 \right\|_{#2}}

\newcommand{\Ne}{{\mathbb{N}}}

\theoremstyle{thmstyleone}%
\newtheorem{theorem}{Theorem}
\newtheorem{proposition}[theorem]{Proposition}%

\theoremstyle{thmstyletwo}%

\theoremstyle{thmstylethree}%
\newtheorem{definition}{Definition}%

\begin{document}

\title[Tensor Train Random Projection]{Tensor Train Random Projection}

\author[1]{\fnm{Yani} \sur{Feng}}\email{fengyn@shanghaitech.edu.cn}
\equalcont{These authors contributed equally to this work.}

\author[2]{\fnm{Kejun} \sur{Tang}}\email{tangkj@pcl.ac.cn}
\equalcont{These authors contributed equally to this work.}

\author[3]{\fnm{Lianxing} \sur{He}}\email{13816474811@139.com}
\author[1]{\fnm{Pingqiang} \sur{Zhou}} \email{zhoupq@shanghaitech.edu.cn}
\author*[1]{\fnm{Qifeng} \sur{Liao}} \email{liaoqf@shanghaitech.edu.cn}

\affil[1]{\orgdiv{School of Information Science and Technology}, \orgname{ShanghaiTech University}, \orgaddress{ \city{Shanghai}, \postcode{200120}, \country{China}}}

\affil[2]{\orgname{Peng Cheng Laboratory}, \orgaddress{\city{Shenzhen}, \postcode{518000}, \country{China}}}

\affil[3]{ \orgname{Innovation Academy of Microsatellite of Chinese Academy of Sciences}, \orgaddress{ \city{Shanghai}, \postcode{201210}, \country{China}}}







\abstract{This work proposes a novel tensor train random projection (TTRP) method for dimension reduction, where pairwise distances
can be approximately preserved.
Our TTRP is systematically constructed through a tensor train (TT) representation with TT-ranks equal to one. 
Based on the tensor train format, 
this new random projection method can  speed up the dimension reduction procedure 
for high-dimensional datasets  
and requires less storage costs with little loss in accuracy, compared with existing methods. 
We provide a theoretical analysis of the bias and the variance of TTRP, which shows that this approach is an expected isometric projection with bounded variance, and we show that the Rademacher distribution is an optimal choice for generating the corresponding TT-cores. 
Detailed numerical experiments with 
synthetic datasets 
and the MNIST dataset  
are conducted to demonstrate the efficiency of TTRP.}




\maketitle

\section{Introduction}\label{Intro}
Dimension reduction is a fundamental concept in science and engineering 
for feature extraction and data visualization.  Exploring the properties of low-dimensional structures in high-dimensional spaces attracts 
broad attention. Popular dimension reduction methods include principal component analysis (PCA) \cite{wold1987principal,vms2016gpca}, non-negative matrix factorization (NMF) \cite{sra2006gnmf}, and t-distributed stochastic neighbor embedding (t-SNE) \cite{maaten2008tsne}. 
A main procedure in dimension reduction is to build a linear or nonlinear mapping from a high-dimensional space to a low-dimensional one, which keeps important properties of the high-dimensional space, such as the distance between any two  points \cite{pham2013fast}. 

The random projection (RP) is a widely used method for dimension reduction. It is well-known that the Johnson-Lindenstrauss (JL) transformation \cite{johnson1984extensions,dasgupta2003elementary} can nearly preserve the distance between two points after a random projection $f$, which is typically called isometry property. The isometry property can be used to achieve the nearest neighbor search in high-dimensional datasets \cite{kleinberg1997nns,ailon2006approximatefastjl}. It can also be used to \cite{baraniuk2008simplerip,krahmer2011newjlrip}, where a sparse signal can be reconstructed under a linear random projection \cite{crt2006robust}. 
The JL lemma \cite{johnson1984extensions}	tells us that there exists a nearly 
	isometry mapping $f$, which maps high-dimensional datasets into a lower dimensional space.
	Typically, a choice for the mapping $f$ is the linear random projection
	\begin{equation}\label{RP}
	f(\mathbf{x}) = \frac{1}{\sqrt{M}} \mathbf{Rx},
	\end{equation}
	where $\mathbf{x}\in \mathbb{R}^N$, and $\mathbf{R} \in \mathbb{R}^{M \times N}$ is a matrix whose entries are drawn from the mean zero and variance one Gaussian distribution, denoted by $\mathcal{N}(0,1)$. We call it Gaussian random projection (Gaussian RP). The storage of matrix $\mathbf{R}$ in \eqref{RP} is $O(MN)$ and the cost of computing $\mathbf{Rx}$ in \eqref{RP} is $O(MN)$. However, with large $M$ and $N$, this construction is computationally infeasible. To alleviate the difficulty, the sparse random projection 
method \cite{achlioptas2003database} and the very sparse random projection method \cite{li2006very} are proposed, where the random projection is constructed by a sparse random matrix. Thus the storage and the computational cost can be reduced. 

To be specific, Achlioptas \cite{achlioptas2003database} replaced the dense matrix $\mathbf{R}$ 
by a sparse matrix whose entries follow
	\begin{equation} \label{sparse_distribution}
	\mathbf{R}_{ij}=\sqrt{s} \cdot
	\begin{cases}
	+1, & \text{with probability} \, \frac{1}{2s}, \\
	0, &  \text{with probability} \,1-\frac{1}{s}, \\
	-1, & \text{with probability} \,\frac{1}{2s}.
	\end{cases}
	\end{equation}
This means that the matrix is sampled at a rate of $1/s$. Note that, if $s=1$, 
the corresponding distribution is called the 
Rademacher distribution. When $s=3$, the cost of computing $\mathbf{Rx}$ in \eqref{RP} reduces down to a third of the original one but is still $O(MN)$. 
When $s=\sqrt{N}\gg3$, Li et al. \cite{li2006very} called this case as {the very sparse random projection} (Very Sparse RP), which significantly speeds up the computation with little loss in accuracy. It is clear that the storage of very sparse random projection is $O(M\sqrt{N})$. However, the sparse random projection can typically distort a sparse vector \cite{ailon2006approximatefastjl}. To achieve a low-distortion embedding, Ailon and Chazelle \cite{ailon2009fastjl,ailon2006approximatefastjl} proposed  the Fast-Johnson-Lindenstrauss Transform (FJLT), where the preconditioning of a sparse projection matrix with a randomized Fourier transform is employed. 
To reduce randomness and storage requirements, Sun \cite{sun2018tensor} et al. proposed the following format:
	$\mathbf{R}=(\mathbf{R}_1\odot\cdots\odot\mathbf{R}_d)^{\text{T}}$,
	where $\odot$ represents the Khatri-Rao product, $\mathbf{R}_i \in \mathbb{R}^{n_i \times M}$, and $N=\prod_{i=1}^{d} n_i$. Each $\mathbf{R}_i$ is a random matrix whose entries are i.i.d. random variables drawn from $\mathcal{N}(0,1)$. This transformation is called {the Gaussian 
 tensor random projection} (Gaussian TRP) throughout this paper. It is clear that the storage of the Gaussian TRP is $O(M\sum_{i=1}^{d}n_i)$, which is less than that of the Gaussian random projection (Gaussian RP) . For example,  when $\mathbf{N}=n_1n_2=40000$, the storage of Gaussian TRP is only $1/20$ of Gaussian RP. Also, it has been shown that Gaussian TRP satisfies the properties of expected isometry with vanishing variance \cite{sun2018tensor}.
 
Recently, using matrix or tensor decomposition to reduce the storage of projection matrices is proposed in \cite{jinfaster,malik2020guarantees}. The main idea of these methods is to split the projection matrix into some small scale matrices or tensors. 
In this work, we focus on the low rank tensor train representation to construct the random projection $f$. Tensor decompositions are 
 widely used for data compression \cite{Kolda2009Tensor,Acar2010Scalable,austin2016paralleltd,pham2013fast,ahle2020oblivious,tang2020rank,cui2021deep}.
 The tensor train (TT) decomposition gives the following benefits---low rank TT-formats can provide compact representations of projection 
 matrices and efficient basic linear algebra operations of matrix-by-vector products \cite{oseledets2011tensor}. 
 Based on these benefits, we propose a novel tensor train random projection (TTRP) method,
 which requires significantly smaller storage and computational costs compared with existing methods (e.g., Gaussian TRP \cite{sun2018tensor}, Very Sparse RP \cite{li2006very} and Gaussian RP \cite{achlioptas2001database}). 
 While constructing projection matrices using   tensor train (TT) and Canonical polyadic (CP) decompositions based on Gaussian random variables is proposed in \cite{rakhshan2020tensorized}, the main contributions of our work are three-fold: first our TTRP is conducted based on a rank-one TT-format, which significantly reduces the storage of projection matrices; second, we provide a novel construction procedure 
 for the rank-one TT-format in our TTRP based on i.i.d.\ Rademacher random variables; 
 third, we prove that our construction of TTRP is unbiased with bounded variance.
 
The rest of the paper is organized as follows. The tensor train format is introduced in section \ref{Talg}. Details of our TTRP approach are introduced in section \ref{TTRP}, where we prove that the approach is an expected isometric projection with bounded variance. In section \ref{Experm}, we demonstrate the efficiency of  TTRP with datasets including synthetic, MNIST. Finally section \ref{Conclu} concludes the paper.
\section{Tensor train format}\label{Talg}
	Let lowercase letters $(x)$, boldface lowercase letters ($\mathbf{x}$), boldface capital letters ($\mathbf{X}$), calligraphy letters $(\mathcal{X})$ be scalar, vector, matrix and tensor variables, respectively.
	$\mathbf{x}(i)$ represents the element $i$ of a vector $\mathbf{x}$.
	$\mathbf{X}(i,j)$ means the element $(i,j)$ of a matrix $\mathbf{X}$. The $i$-th row and $j$-th column of a matrix $\mathbf{X}$ is defined by $\mathbf{X}(i,:)$ and $\mathbf{X}(:,j)$, respectively. For a given $d$-th order tensor $\mathcal{X}$, $\mathcal{X}({i_1, i_2, \ldots, i_d})$ is its $({i_1, i_2, \ldots, i_d})$-th component. For a vector $\mathbf{x}\in \mathbb{R}^N$, we denote its $\ell^{p}$ norm as ${\Arrowvert\mathbf{x}\Arrowvert}_p=(\sum_{i=1}^{N}{\lvert\mathbf{x}(i)\rvert}^p)^{\frac{1}{p}}$, for any $p\geq 1$. 
	   The Kronecker product of matrices $\mathbf{A}\in \mathbb{R}^{I \times J}$ and $\mathbf{B}\in \mathbb{R}^{K \times L}$ is denoted by $\mathbf{A}\otimes\mathbf{B}$ of which the result is a matrix of size $(IK)\times (JL)$ and defined by
	\begin{equation*}
	\mathbf{A}\otimes\mathbf{B}=\left[
	\begin{array}{cccc}
	\mathbf{A}(1,1)\mathbf{B} & \mathbf{A}(1,2)\mathbf{B} & \cdots &\mathbf{A}(1,J)\mathbf{B}\\
	\mathbf{A}(2,1)\mathbf{B}& \mathbf{A}(2,2)\mathbf{B} & \cdots&\mathbf{A}(2,J)\mathbf{B}\\
	\vdots & \vdots& \ddots & \vdots\\
	\mathbf{A}(I,1)\mathbf{B}&\mathbf{A}(I,2)\mathbf{B}&\cdots&\mathbf{A}(I,J)\mathbf{B}
	\end{array}
	\right].
	\end{equation*}
	The Kronecker product conforms the following laws \cite{van2000ubiquitous}:
	\begin{equation}\label{multiply}
	(\mathbf{A}\mathbf{C})\otimes(\mathbf{B}\mathbf{D})=(\mathbf{A}\otimes\mathbf{B})(\mathbf{C}\otimes\mathbf{D}),
	\end{equation}
	\begin{equation}\label{distributive}
	(\mathbf{A}+\mathbf{B})\otimes(\mathbf{C}+\mathbf{D})=\mathbf{A}\otimes\mathbf{C}+ \mathbf{A}\otimes\mathbf{D}+\mathbf{B}\otimes\mathbf{C}+\mathbf{B}\otimes\mathbf{D},
	\end{equation}
	\begin{equation}
	    \left(k\mathbf{A}\right)\otimes \mathbf{B}=\mathbf{A}\otimes \left(k \mathbf{B}\right)=k\left(\mathbf{A}\otimes\mathbf{B}\right).
	\end{equation}

\subsection{Tensor train decomposition}
Tensor Train (TT) decomposition \cite{oseledets2011tensor} is a generalization of SVD decomposition from matrices to tensors. TT decomposition provides a compact representation for tensors, and allows for efficient application of linear algebra operations (discussed in section \ref{matrix-vector} and section \ref{basic}).

	Given a $d$-th order tensor $\mathcal{G} \in \mathbb{R}^{n_1 \times \cdots \times n_d}$, the tensor train decomposition \cite{oseledets2011tensor} is 
	\begin{equation}\label{TTformat_ele}
	\mathcal{G}({i_1, i_2, \ldots, i_d}) = \mathcal{G}_1(i_1) \mathcal{G}_2(i_2) \cdots \mathcal{G}_d(i_d),
	\end{equation}
where
$\mathcal{G}_k \in \mathbb{R}^{r_{k-1}\times n_k \times r_k}$ are called TT-cores, $\mathcal{G}_k(i_k) \in \mathbb{R}^{r_{k-1} \times r_k}$ is a slice of $\mathcal{G}_k$, for $k=1,2,\ldots,d$, $i_k = 1, \ldots, n_k$, and the ``boundary condition'' is $r_0 = r_d = 1$. The tensor $\mathcal{G}$ is said to be in the TT-format if each element of $\mathcal{G}$  can be represented by \eqref{TTformat_ele}.
The vector $[r_0,r_1,r_2, \ldots, r_d]$ is referred to as TT-ranks. Let $\mathcal{G}_k(\alpha_{k-1}, i_k, \alpha_{k})$ represent the element of $\mathcal{G}_k(i_k)$ in the position $(\alpha_{k-1}, \alpha_k)$. In the index form, the decomposition (\ref{TTformat_ele}) is rewritten as the following TT-format
	\begin{equation}\label{TTformat_core}
	\mathcal{G}(i_1,i_2,\ldots,i_d) = \sum\limits_{\alpha_0,\cdots,\alpha_d}{\mathcal{G}_1(\alpha_0,i_1,\alpha_1)\mathcal{G}_2(\alpha_1,i_2,\alpha_2)\cdots \mathcal{G}_d(\alpha_{d-1},i_d,\alpha_d)}. 
	\end{equation}
	
	 To look more closely to \eqref{TTformat_ele}, an element $\mathcal{G}(i_1, i_2, \ldots, i_d)$ is represented by a sequence of matrix-by-vector products.
	Figure \ref{ttformat_fig} illustrates the tensor train decomposition. It can be seen that the key ingredient in tensor train (TT) decomposition is the TT-ranks. The TT-format only  uses $O(ndr^2)$ memory to $O(n^d)$ elements, where $n = \text{max} \ \{n_1, \ldots, n_d\}$ and $r = \text{max}\ \{r_0, r_1, \ldots, r_d\}$. Although the storage reduction is efficient only if the TT-rank is small, tensors in data science and machine learning typically have low TT-ranks. Moreover, one can apply the TT-format to basic linear algebra operations, such as matrix-by-vector products, scalar multiplications, etc. This can reduce the computational cost significantly when the data have low rank structures  (see \cite{oseledets2011tensor} for details). 
	\begin{figure}
	    \centering
	    \includegraphics[width=.9\textwidth]{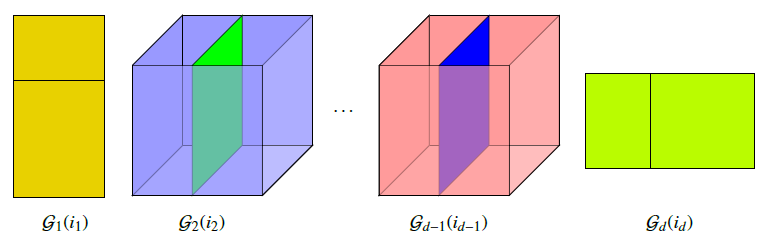}
	    \caption{Tensor train format (TT-format): extract an element $\mathcal{G}({i_1, i_2, \ldots, i_d})$ via a sequence of matrix-by-vector products.}
	    \label{ttformat_fig}
	\end{figure}

	\subsection{Tensorizing matrix-by-vector products}\label{matrix-vector}
	
	The tensor train format gives a compact representation of matrices and efficient computation for matrix-by-vector products.
	We first review the TT-format of large matrices and vectors following \cite{oseledets2011tensor}. Defining two bijections $\nu: \mathbb{N} \mapsto \mathbb{N}^{d}$ and $\mu: \mathbb{N} \mapsto \mathbb{N}^{d}$, a pair index $(i, j) \in \mathbb{N}^{2}$ is mapped to a multi-index pair $(\nu (i), \mu (j)) = (i_1. i_2, \ldots, i_d, j_1, j_2, \ldots, j_d)$. 
Then a matrix $\mathbf{R} \in \mathbb{R}^{M \times N}$ and a vector $\mathbf{x} \in \mathbb{R}^{N}$ can be 
tensorized in the TT-format as follows. 
Letting $M = \prod_{i=1}^d {m_k}$ and $N = \prod_{i=1}^d {n_k}$, an element $(i,j)$ of $\mathbf{R}$ can be written as 
(see \cite{novikov2015tensorizing,oseledets2011tensor})
	\begin{equation} \label{eq_tensorizeR}
	\mathbf{R}(i,j) = \mathcal{R}(\nu (i), \mu (j)) = \mathcal{R}(i_1,\dots,i_d,j_1,\dots, j_d)=  \mathcal{R}_1(i_1,j_1)\cdots\mathcal{R}_d(i_d,j_d),
	\end{equation} 
	and an element $j$ of $\mathbf{x}$ can be written as 
	\begin{equation} \label{eq_tensorizeX}
	\mathbf{x}(j) = \mathcal{X}(\mu (j)) = \mathcal{X}(j_1,\dots,j_d)=\mathcal{X}_1(j_1)\cdots\mathcal{X}_d(j_d),
	\end{equation}
	where $\mathcal{R}_k(i_k,j_k) \in \mathbb{R}^{r_{k-1} \times r_k},\,\mathcal{X}_k(j_k)\in \mathbb{R}^{\hat{r}_{k-1} \times \hat{r}_k},\,r_0=\hat{r}_0=r_d=\hat{r}_d=1$, for $k=1,\dots,d$, $(i_1,\dots i_d)$ enumerate the rows of $\mathbf{R}$,
	and $(j_1,\dots, j_d)$ enumerate the columns of $\mathbf{R}$. We consider the matrix-by-vector product ($\mathbf{y}=\mathbf{R}\mathbf{x}$), and each element of $\mathbf{y}$ can be tensorized in the TT-format as
	\begin{equation}\label{ycore}
	\begin{aligned}
	\mathbf{y}(i)=\mathcal{Y}(i_1,\dots,i_d)=&\sum_{j_1,\dots,j_d} \mathcal{R}(i_1,\dots,i_d,j_1,\dots, j_d)\mathcal{X}(j_1,\dots,j_d)\\
	=&\sum_{j_1,\dots,j_d}\Big({\mathcal{R}_1(i_1,j_1)}\cdots\mathcal{R}_d(i_d,j_d)\Big)\Big(\mathcal{X}_1(j_1)\cdots\mathcal{X}_d(j_d)\Big)\\
	=&\sum_{j_1,\dots,j_d}\underbrace{\Big(\mathcal{R}_1(i_1,j_1)\otimes \mathcal{X}_1(j_1)\Big)}_{O(r_0r_1\hat{r}_0\hat{r}_1)}\cdots\underbrace{\Big(\mathcal{R}_d(i_d,j_d)\otimes \mathcal{X}_d(j_d)\Big)}_{O(r_{d-1}r_d\hat{r}_{d-1}\hat{r}_d)}\\
	=&\underbrace{\mathcal{Y}_1(i_1)}_{O(n_1r_0r_1\hat{r}_0\hat{r}_1)}\cdots\underbrace{\mathcal{Y}_d(i_d)}_{O(n_dr_{d-1}r_d\hat{r}_{d-1}\hat{r}_d)},
	\end{aligned}
	\end{equation}
where $\mathcal{Y}_k(i_k)=\sum_{j_k}\mathcal{R}_k(i_k,j_k)\otimes \mathcal{X}_k(j_k)\in \mathbb{R}^{r_{k-1}\hat{r}_{k-1}\times r_k\hat{r}_k}$, for $k=1,\dots,d$. The complexity of computing each TT-core $\mathcal{Y}_k \in \mathbb{R}^{r_{k-1}\hat{r}_{k-1}\times m_k\times r_k \hat{r}_k}$, is $O(m_kn_k r_{k-1}r_k \hat{r}_{k-1}\hat{r}_k)$ for $k=1,\dots,d$.
Assuming that the TT-cores of $\mathbf{x}$ are known, the total cost of the matrix-by-vector product ($\mathbf{y}=\mathbf{R}\mathbf{x}$) in the TT-format can reduce significantly from the original complexity $O(MN)$ to $O(dmnr^2\hat{r}^2),\,m=\max\{m_1,m_2,\dots,m_d\}$, $n=\max\{n_1,n_2,\dots,n_d\},\,r=\text{max}\ \{r_0, r_1, \ldots, r_d\},\,\hat{r}=\text{max}\ \{\hat{r}_0, \hat{r}_1, \ldots, \hat{r}_d\}$, where $N$ is typically large and $r$ is small. When $m_k=n_k,\,r_k=\hat{r}_k$, for $k=1,\dots,d$, the cost of such matrix-by-vector product in the TT-format is $O(dn^2r^4)$ \cite{oseledets2011tensor}. Note that, in the case that $r$ equals one, the cost of such matrix-by-vector product in the TT-format is $O(dmn\hat{r}^2)$.

\subsection{Basic Operations in the TT-format}\label{basic}
  In section \ref{matrix-vector}, the product of matrix $\mathbf{R}$ and vector $\mathbf{x}$ which are both in the TT-format, is conducted efficiently.
  In the TT-format, many important operations can be readily implemented. For instance, computing  the Euclidean distance between two vectors in the TT-format is more efficient with less storage than directly computing the Euclidean distance in standard matrix and vector formats. In the following, some important operations in the TT-format are discussed. 
  
 The subtraction of tensor $\mathcal{Y}\in\mathbb{R}^{m_1\times\cdots \times m_d }$ and tensor $\hat{\mathcal{Y}}\in\mathbb{R}^{m_1\times\cdots \times m_d }$ in the TT-format is
 \begin{equation} \label{add}
     \begin{aligned}
     \mathcal{Z}(i_1,\dots,i_d)
     &:=\mathcal{Y}(i_1,\dots,i_d)-\hat{\mathcal{Y}}(i_1,\dots,i_d)\\
     &=\mathcal{Y}_1(i_1)\mathcal{Y}_2(i_2)\cdots\mathcal{Y}_d(i_d)-\hat{\mathcal{Y}}_1(i_1)\hat{\mathcal{Y}}_2(i_2)\cdots\hat{\mathcal{Y}}_d(i_d)\\
     &=\mathcal{Z}_1(i_1)\mathcal{Z}_2(i_2)\cdots\mathcal{Z}_d(i_d),
 \end{aligned}
 \end{equation}
 where 
 \begin{equation}
    \mathcal{Z}_{k}\left(i_{k}\right)=\left(\begin{array}{cc}
\mathcal{Y}_{k}\left(i_{k}\right) & 0 \\
0 & \hat{\mathcal{Y}}_{k}\left(i_{k}\right)
\end{array}\right), \quad k=2, \ldots, d-1, 
 \end{equation}
and
\begin{equation}
   \mathcal{Z}_{1}\left(i_{1}\right)=\left(\begin{array}{cc}
\mathcal{Y}_{1}\left(i_{1}\right) & -\hat{\mathcal{Y}}_{1}\left(i_{1}\right)
\end{array}\right), \quad \mathcal{Z}_{d}\left(i_{d}\right)=\left(\begin{array}{c}
\mathcal{Y}_{d}\left(i_{d}\right) \\
\hat{\mathcal{Y}}_{d}\left(i_{d}\right)
\end{array}\right), 
\end{equation}
and TT-ranks of $\mathcal{Z}$ equal the sum of TT-ranks of $\mathcal{Y}$ and $\hat{\mathcal{Y}}$.

The dot product of 
tensor $\mathcal{Y}$ and tensor $\hat{\mathcal{Y}}$ in the TT-format \cite{oseledets2011tensor} is
\begin{equation}\label{dot1}
\begin{aligned}
\langle\mathcal{Y}, \hat{\mathcal{Y}}\rangle &:=\sum_{i_{1}, \ldots, i_{d}} \mathcal{Y}\left(i_{1}, \ldots, i_{d}\right) \hat{\mathcal{Y}}\left(i_{1}, \ldots, i_{d}\right)\\
     &=\sum_{i_{1}, \ldots, i_{d}}\Big(\mathcal{Y}_1(i_1)\mathcal{Y}_2(i_2)\cdots\mathcal{Y}_d(i_d)\Big)\Big(\hat{\mathcal{Y}}_1(i_1)\hat{\mathcal{Y}}_2(i_2)\cdots\hat{\mathcal{Y}}_d(i_d)\Big)\\
     &=\sum_{i_{1}, \ldots, i_{d}}\Big(\mathcal{Y}_1\left(i_1\right)\mathcal{Y}_2\left(i_2\right)\cdots\mathcal{Y}_d(i_d)\Big)\otimes\Big(\hat{\mathcal{Y}}_1(i_1)\hat{\mathcal{Y}}_2(i_2)\cdots\hat{\mathcal{Y}}_d(i_d)\Big)\\
     &=\left(\sum_{i_{1}}\mathcal{Y}_{1}\left(i_{1}\right) \otimes \hat{\mathcal{Y}}_{1}\left(i_{1}\right)\right)\left(\sum_{i_{2}}\mathcal{Y}_{2}\left(i_{2}\right) \otimes \hat{\mathcal{Y}}_{2}\left(i_{2}\right)\right) \ldots\left(\sum_{i_{d}}\mathcal{Y}_{d}\left(i_{d}\right) \otimes \hat{\mathcal{Y}}_{d}\left(i_{d}\right)\right)\\
     &=\mathbf{V}_1\mathbf{V}_2\cdots\mathbf{V}_d,
\end{aligned}
\end{equation}
where \begin{equation}\label{dot2}
    \mathbf{V}_k=\sum_{i_{k}}\mathcal{Y}_{k}\left(i_{k}\right) \otimes \hat{\mathcal{Y}}_{k}\left(i_{k}\right), \quad k=1, \ldots, d.
\end{equation}
 Since $\mathbf{V}_1,\mathbf{V}_d$ are vectors and $\mathbf{V}_2,\dots,\mathbf{V}_{d-1}$ are matrices, we compute $\langle\mathcal{Y}, \hat{\mathcal{Y}}\rangle$ by a sequence of matrix-by-vector products:
\begin{align}
  \mathbf{v_{1}}&=\mathbf{V}_{1}, \label{v1}\\
    \mathbf{v_{k}}=\mathbf{v_{k-1}} \mathbf{V_k}=\mathbf{v_{k-1}} \sum_{i_{k}} \mathcal{Y}_{k}\left(i_{k}\right) \otimes \hat{\mathcal{Y}}_{k}\left(i_{k}\right)&=\sum_{i_{k}} \mathbf{p_{k}}\left(i_{k}\right), \quad k=2, \ldots, d,  \label{second}
\end{align}
where
\begin{align}\label{first}
    \mathbf{p_{k}}\left(i_{k}\right)=\mathbf{v_{k-1}}\left(\mathcal{Y}_{k}\left(i_{k}\right) \otimes \hat{\mathcal{Y}}_{k}\left(i_{k}\right)\right),
\end{align}
and we obtain
\begin{equation}
    \langle\mathcal{Y}, \hat{\mathcal{Y}}\rangle=\mathbf{v_d}.
\end{equation}
For simplify we assume that TT-ranks of $\mathcal{Y}$ are the same as that of $\hat{\mathcal{Y}}$.
In \eqref{first}, let $\mathbf{B}:=\mathcal{Y}_k(i_k)\in \mathbb{R}^{r\times r},\,\mathbf{C}:=\hat{\mathcal{Y}}_k(i_k)\in \mathbb{R}^{r\times r},\,\mathbf{x}:=\mathbf{v_{k-1}}\in \mathbb{R}^{1\times r^2},\,\mathbf{y}:= \mathbf{p_{k}\left(i_{k}\right)}\in \mathbb{R}^{1\times r^2}$, for $k=2,\dots,d-1$, and we use the reshaping Kronecker product expressions \cite{golub2013matrix} for \eqref{first}:
$$\mathbf{y}=\mathbf{x}(\mathbf{B}\otimes\mathbf{C})\quad \Longleftrightarrow \quad \mathbf{Y}=\mathbf{C}^{T}\mathbf{X}\mathbf{B},$$
where we reshape $\mathbf{x},\,\mathbf{y}$ into $X=\left[\begin{array}{llll}\mathbf{x}_{1} & \mathbf{x}_{2} & \cdots & \mathbf{x}_{r}\end{array}\right] \in \mathbb{R}^{r \times r}$, $ Y=\left[\begin{array}{llll}\mathbf{y}_{1} & \mathbf{y}_{2} & \cdots & \mathbf{y}_{r}\end{array}\right] \in \mathbb{R}^{r \times r}$ respectively. Note that the cost of computing $\mathbf{Y}=\mathbf{C}^{T}\mathbf{X}\mathbf{B}$ is $O(r^3)$ while the disregard of Kronecker structure of $\mathbf{y}=\mathbf{x}(\mathbf{B}\otimes\mathbf{C})$ leads to an $O(r^4)$ calculation. Hence the complexity of computing $\mathbf{p_{k}\left(i_{k}\right)}$ in \eqref{first} is $O(r^3)$, because of the efficient Kronecker product computation. Then the cost of computing $\mathbf{v_k}$ in \eqref{second}  is $O(mr^3)$, and the total cost of the dot product $\langle\mathcal{Y}, \hat{\mathcal{Y}}\rangle$ is $O(dmr^3)$. 

The Frobenius norm of a tensor $\mathcal{Y}$ is defined by
$$
\norm{\mathcal{Y}}{F}=\sqrt{\langle\mathcal{Y}, \mathcal{Y}\rangle}.
$$
Computing the distance between tensor $\mathcal{Y}$ and tensor $\hat{\mathcal{Y}}$ in the TT-format is computationally efficient by applying the dot product \eqref{dot1}--\eqref{dot2},
\begin{equation} \label{dis}
     \norm{\mathcal{Y}-\hat{\mathcal{Y}}}{F}=\sqrt{\langle \mathcal{Y}-\hat{\mathcal{Y}}, \mathcal{Y}-\hat{\mathcal{Y}}\rangle}.
\end{equation}
The complexity of computing the distance is also $O(dmr^3)$. Algorithm \ref{dot} gives more details about computing \eqref{dis} based on Frobenius norm $\norm{\mathcal{Y}-\hat{\mathcal{Y}}}{F}$.
\begin{algorithm}[H]
	\caption{Distance based on Frobenius Norm $W:=\norm{\mathcal{Y}-\hat{\mathcal{Y}}}{F}=\sqrt{\langle \mathcal{Y}-\hat{\mathcal{Y}}, \mathcal{Y}-\hat{\mathcal{Y}}\rangle}$}
	\label{dot}
	\begin{algorithmic}[1]
		\Require TT-cores $\mathcal{Y}_k$ of tensor $\mathcal{Y}$ and TT-cores $\hat{\mathcal{Y}}_k$ of tensor $\hat{\mathcal{Y}}$, for $k=1,\dots,d$.
		\State Compute $\mathcal{Z}:=\mathcal{Y}-\hat{\mathcal{Y}}.$ 	$\qquad \qquad \qquad \qquad \qquad \qquad \qquad \triangleright \ O(mr)$ by \eqref{add}
		\State Compute $\mathbf{v_1}:=\sum_{i_{1}}\mathcal{Z}_{1}\left(i_{1}\right) \otimes \mathcal{Z}_{1}\left(i_{1}\right)$.$\qquad \qquad \qquad \qquad \quad \triangleright \ O(mr^2)$ by \eqref{v1}
		\For {$k = 2:d-1$}
		\State Compute $\mathbf{p_{k}}\left(i_{k}\right)=\mathbf{v_{k-1}}\Big(\mathcal{Z}_{k}(i_{k}) \otimes \mathcal{Z}_{k}(i_{k})\Big)$. $\qquad \qquad  \quad \triangleright \ O(r^3)$ by \eqref{first}
		\State Compute $\mathbf{v_{k}}:=\sum_{i_{k}} \mathbf{p_{k}}\left(i_{k}\right)$. $\qquad \qquad \qquad \qquad \qquad \qquad \triangleright \ O(mr^3)$ by \eqref{second}
		\EndFor
		\State Compute $\mathbf{p_{d}}\left(i_{d}\right)=\mathbf{v_{d-1}}\Big(\mathcal{Z}_{d}(i_{d}) \otimes \mathcal{Z}_{d}\left(i_{d}\right)\Big)$. $\qquad \qquad  \qquad  \triangleright \ O(r^2)$ by \eqref{first}
		\State Compute $\mathbf{v_{d}}:=\sum_{i_{d}} \mathbf{p_{d}}(i_{d})$. $\qquad \qquad \qquad \qquad \qquad \qquad \quad  \triangleright \ O(mr^2)$ by \eqref{second}
		\Ensure Distance $W:=\sqrt{\langle \mathcal{Y}-\hat{\mathcal{Y}}, \mathcal{Y}-\hat{\mathcal{Y}}\rangle}=\sqrt{\mathbf{v_d}}$.
	\end{algorithmic}
\end{algorithm}

In summary, just merging the cores of two tensors in the TT-format can perform the subtraction of two tensors instead of directly subtraction of two tensors in standard tensor format. A sequence of matrix-by-vector products can achieve the dot product of two tensors in the TT-format. The cost of computing the distance between two tensors in the TT-format, reduces from the original complexity $O(M )$
to $O(dmr^3)$, where $M=\prod_{i=1}^{d}m_i,\,r\ll M $.

\section{Tensor train random projection}\label{TTRP}
Due to the computational efficiency of TT-format discussed above, 
we consider the TT-format to construct projection matrices. Our tensor train random projection is defined as follows.
\begin{definition}\label{defTTRP}
(Tensor Train Random Projection). For a data point $\mathbf{x} \in \dsR^{N}$, our tensor train random projection (TTRP) is 
\begin{equation} \label{eq_ttfpfull}
	f_{TTRP}(\mathbf{x}):= \frac{1}{\sqrt{M}}\mathbf{R} \mathbf{x},
	\end{equation}
where the tensorized versions (through 
the TT-format) of 
$\mathbf{R}$ and $\mathbf{x}$ are denoted by $\mathcal{R}$ and 
$\mathcal{X}$ (see \eqref{eq_tensorizeR}-\eqref{eq_tensorizeX}),
the corresponding TT-cores
are denoted by $\{\mathcal{R}_k \in \dsR^{r_{k-1}\times m_k\times n_k \times r_k}\}^d_{k=1}$ and $\{\mathcal{X}_k\in \mathbb{R}^{\hat{r}_{k-1}\times n_k \times \hat{r}_k}\}^d_{k=1}$ respectively, we set  $r_0=r_1=\ldots=r_d=1$,
and $\mathbf{y}:=\mathbf{R}\mathbf{x}$ is specified by \eqref{ycore}. 
\end{definition}
Note that our TTRP is based on the tensorized version of $\mathbf{R}$ with TT-ranks all equal to one, 
which leads to significant computational efficiency and small storage costs,  and comparisons for TTRP associated with different TT-ranks are conducted in section \ref{Experm}.
When $r_0=r_1=\ldots=r_d=1$, 
all TT-cores $\mathcal{R}_i$, for $i=1,\dots,d$ in \eqref{eq_tensorizeR} become matrices and the cost of computing $\mathbf{Rx}$ in TTRP \eqref{eq_ttfpfull} is $O(dmn\hat{r}^2)$ (see section \ref{matrix-vector}), where $m=\max\{m_1,m_2,\dots,m_d\}$, $n=\max\{n_1,n_2,\dots,n_d\}$ and $\hat{r}=\max\{\hat{r}_0, \hat{r}_1, \ldots, \hat{r}_d\}$. 
Moreover, from our analysis in the latter part of this section, we find that the Rademacher distribution introduced in section \ref{Intro} is an optimal choice for generating
the TT-cores of $\mathbf{R}$.
In the following, 
we prove that TTRP established by \eqref{eq_ttfpfull} is an expected isometric projection with bounded variance. 

\begin{theorem}\label{lemma_mean}
		Given a vector $\mathbf{x}\in\mathbb{R}^{\prod_{j=1}^{d} n_j}$, if $\mathbf{R}$ in \eqref{eq_ttfpfull} is composed of $d$ independent TT-cores $\mathcal{R}_1,\dots,\mathcal{R}_d$, whose entries are independent and identically random variables with mean zero and variance one, then the following equation holds
		\begin{equation*}
		\mathbb{E}{\Arrowvert f_{TTRP}(\mathbf{x})\Arrowvert}^{2}_2={\Arrowvert \mathbf{x} \Arrowvert}^2_2.
		\end{equation*}
\end{theorem}
	\begin{proof}
Denoting $\mathbf{y}=\mathbf{R}\mathbf{x}$ gives
\begin{align}
    \mathbb{E}{\Arrowvert f_{TTRP}(\mathbf{x})\Arrowvert}^2_2=\frac{1}{M}\mathbb{E}{\Arrowvert \mathbf{y} \Arrowvert}^2_2=\frac{1}{M}\mathbb{E}\left[\sum_{i=1}^M \mathbf{y}^2(i)\right]=\frac{1}{M}\mathbb{E}\left[\sum_{i_1,\dots,i_d} \mathcal{Y}^2(i_1,\dots,i_d)\right] . \label{yintial}
\end{align}
			By the TT-format, 
			$\mathcal{Y}(i_1,\dots,i_d)=\mathcal{Y}_1(i_1)\cdots\mathcal{Y}_d(i_d)$, where $\mathcal{Y}_k(i_k)=\sum_{j_k}\mathcal{R}_k(i_k,j_k)\otimes \mathcal{X}_k(j_k)$, for $k=1,\dots,d$, it follows that
			\begin{align}
				\mathbb{E}\left[\mathcal{Y}^2(i_1,\dots,i_d)\right]&=\mathbb{E}\left[\left(\mathcal{Y}_1(i_1)\cdots\mathcal{Y}_d(i_d)\Big)\Big(\mathcal{Y}_1(i_1)\cdots\mathcal{Y}_d(i_d)\right)\right] \nonumber\\
				&=\mathbb{E}\left[\Big(\mathcal{Y}_1(i_1)\cdots\mathcal{Y}_d(i_d)\Big)\otimes\Big(\mathcal{Y}_1(i_1)\cdots\mathcal{Y}_d(i_d)\Big)\right] \label{y2}\\
				&=\mathbb{E}\left[\Big(\mathcal{Y}_1(i_1)\otimes\mathcal{Y}_1(i_1)\Big)\cdots\Big(\mathcal{Y}_d(i_d)\otimes\mathcal{Y}_d(i_d)\Big)\right] \label{y3}\\
				&=\mathbb{E}\Big[\mathcal{Y}_1(i_1)\otimes\mathcal{Y}_1(i_1)\Big]\cdots\mathbb{E}\Big[\mathcal{Y}_d(i_d)\otimes\mathcal{Y}_d(i_d)\Big], \label{yy}
			\end{align}
			where \eqref{y3} is derived using  \eqref{multiply} and \eqref{y2}, and then combining \eqref{y3}
			and using the independence of TT-cores $\mathcal{R}_1,\dots,\mathcal{R}_d$ give \eqref{yy}.

The $k$-th term of the right hand side of \eqref{yy}, for $k=1,\dots,d$, can be computed by
			\begin{align}
				\mathbb{E}\Big[\mathcal{Y}_k(i_k)\otimes\mathcal{Y}_k(i_k)\Big]  &=\mathbb{E}\Bigg[\Big[\sum_{j_k}\mathcal{R}_k(i_k,j_k)\otimes\mathcal{X}_k(j_k)\Big]\otimes\Big[\sum_{j_k}\mathcal{R}_k(i_k,j_k)\otimes\mathcal{X}_k(j_k)\Big]\Bigg] \label{yk1}\\
				&=\mathbb{E}\Bigg[\Big[\sum_{j_k}\mathcal{R}_k(i_k,j_k)\mathcal{X}_k(j_k)\Big]\otimes\Big[\sum_{j_k}\mathcal{R}_k(i_k,j_k)\mathcal{X}_k(j_k)\Big]\Bigg]\label{yk2}\\
				&=\sum_{j_k,j'_k}\mathbb{E}\Big[\mathcal{R}_k(i_k,j_k)\mathcal{R}_k(i_k,j'_k)\Big]\mathcal{X}_k(j_k)\otimes\mathcal{X}_k(j'_k) \label{yk3}\\
				&=\sum_{j_k}\mathbb{E}\Big[\mathcal{R}^2_k(i_k,j_k)\Big]\mathcal{X}_k(j_k)\otimes\mathcal{X}_k(j_k) \label{yk4}\\
				&=\sum_{j_k}\mathcal{X}_k(j_k)\otimes\mathcal{X}_k(j_k). \label{yk5}
\end{align}
Here as we set the TT-ranks
of $\mathcal{R}$ to be one,   $\mathcal{R}_k(i_k,j_k)$ is  scalar, and  \eqref{yk1} then
leads to \eqref{yk2}. 
Using \eqref{distributive} and \eqref{yk2} gives \eqref{yk3}, and we derive \eqref{yk5} from \eqref{yk3} by the assumption that $\mathbb{E}\Big[\mathcal{R}^2_k (i_k,j_k)\Big]=1$ and $\mathbb{E}\Big[\mathcal{R}_k(i_k,j_k)\mathcal{R}_k(i_k,j'_k)\Big]=0$, for $ j_k, j'_k=1,\dots,n_k$, $j_k\neq j'_k$, $k=1,\dots,d$.

Substituting \eqref{yk5} into \eqref{yy} gives
			\begin{align}
				\mathbb{E}\Big[\mathcal{Y}^2(i_1,\dots,i_d)\Big]&=\Bigg[\sum_{j_1}\mathcal{X}_1(j_1)\otimes\mathcal{X}_1(j_1)\Bigg]\cdots\Bigg[\sum_{j_d}\mathcal{X}_d(j_d)\otimes\mathcal{X}_d(j_d)\Bigg] \nonumber\\
				&=\sum_{j_1,\dots,j_d}\Big[\mathcal{X}_1(j_1)\cdots\mathcal{X}_d(j_d)\Big]\otimes\Big[\mathcal{X}_1(j_1)\cdots\mathcal{X}_d(j_d)\Big] \nonumber\\
				&=\sum_{j_1,\dots,j_d}\mathcal{X}^2(j_1,\dots,j_d) \nonumber\\
				&=\norm{\mathbf{x}}{2}^2. \label{yresult}
			\end{align}
			Substituting \eqref{yresult} into \eqref{yintial}, it concludes that 
			\begin{align*}
				\mathbb{E}{\Arrowvert f_{TTRP}(\mathcal{X})\Arrowvert}^2_2 &=\frac{1}{M}\mathbb{E}\Bigg[\sum_{i_1,\dots,i_d} \mathcal{Y}^2(i_1,\dots,i_d)\Bigg]\\
				&=\frac{1}{M}\times M {\Arrowvert \mathbf{x} \Arrowvert}^2_2\\
				&={\Arrowvert \mathbf{x} \Arrowvert}^2_2.
			\end{align*}
\end{proof}
	
	\begin{theorem}\label{lemma_var}
		Given a vector $\mathbf{x} \in \mathbb{R}^{\prod_{j=1}^{d} n_j}$, if $\mathbf{R}$ in \eqref{eq_ttfpfull} is composed of $d$ independent TT-cores $\mathcal{R}_1,\dots,\mathcal{R}_d$, whose entries are independent and identically random variables with mean zero, variance one, with the same fourth moment $\Delta$ and $\mathcal{M}:=\max_{i=1,\dots,N} \ \lvert\mathbf{x}(i)\rvert,\,m=\max\{m_1,m_2,\dots,m_d\},\, n=\max\{n_1,n_2,\dots,n_d\}$, then
	$$
		\text{Var}\left({\Arrowvert f_{TTRP}(\mathbf{x})\Arrowvert}^{2}_2 \right) \leq \frac{1}{M}\Big(\Delta+n(m+2)-3\Big)^d N\mathcal{M}^4-{\Arrowvert \mathbf{x} \Arrowvert}^4_2.
	$$
	\end{theorem}
	\begin{proof}
		By the property of the variance and using Theorem \ref{lemma_mean}, 
			\begin{align}
				\text{Var}\left({\Arrowvert f_{TTRP}(\mathbf{x}\Arrowvert}^{2}_2\right)&=\mathbb{E}\Big[{\Arrowvert f_{TTRP}(\mathbf{x})\Arrowvert}^{4}_2\Big]-\Bigg[\mathbb{E}\Big[{\Arrowvert f_{TTRP}(\mathbf{x})\Arrowvert}^{2}_2\Big]\Bigg]^2\nonumber\\
				&=\mathbb{E}\Big[\Arrowvert \frac{1}{\sqrt{M}}\mathbf{y}\Arrowvert^4_2\Big]-{\Arrowvert \mathbf{x} \Arrowvert}^4_2 \nonumber\\
				&= \frac{1}{M^2}\mathbb{E}\Big[{\Arrowvert\mathbf{y}\Arrowvert}^{4}_2\Big]-{\Arrowvert \mathbf{x} \Arrowvert}^4_2 \label{var_right}\\
				&=\frac{1}{M^2}\Bigg[\sum_{i=1}^M \mathbb{E}\Big[\mathbf{y}^4(i)\Big]+\sum_{i\neq j}\mathbb{E}\Big[\mathbf{y}^2(i){\mathbf{y}^2(j)}\Big]\Bigg]-{\Arrowvert \mathbf{x} \Arrowvert}^4_2, \label{square_right}
			\end{align}
			where note that $\mathbb{E}[\mathbf{y}^2(i)\mathbf{y}^2(j)]\neq \mathbb{E}[\mathbf{y}^2(i)]\mathbb{E}[\mathbf{y}^2(j)]$ in general and a simple example can be found in Appendix A.
		
			We compute the first term of the right hand side of \eqref{square_right},
    	\begin{align}
			\mathbb{E}\Big[\mathbf{y}^4(i)\Big]
			&=\mathbb{E}\Big[\mathcal{Y}(i_1,\dots,i_d)\otimes\mathcal{Y}(i_1,\dots,i_d)\otimes\mathcal{Y}(i_1,\dots,i_d)\otimes\mathcal{Y}(i_1,\dots,i_d)\Big] \label{yyyy1}\\
			&=\mathbb{E}\Bigg[\Big[\mathcal{Y}_1(i_1)\otimes\mathcal{Y}_1(i_1)\otimes\mathcal{Y}_1(i_1)\otimes\mathcal{Y}_1(i_1)\Big]\cdots\Big[\mathcal{Y}_d(i_d)\otimes\mathcal{Y}_d(i_d)\otimes\mathcal{Y}_d(i_d)\otimes\mathcal{Y}_d(i_d)\Big]\Bigg] \label{yyyy2}\\
		&=\mathbb{E}\Big[\mathcal{Y}_1(i_1)\otimes\mathcal{Y}_1(i_1)\otimes\mathcal{Y}_1(i_1)\otimes\mathcal{Y}_1(i_1)\Big]\cdots\mathbb{E}\Big[\mathcal{Y}_d(i_d)\otimes\mathcal{Y}_d(i_d)\otimes\mathcal{Y}_d(i_d)\otimes\mathcal{Y}_d(i_d)\Big] \label{yyyy3},
			\end{align}
			where $\mathbf{y}(i)=\mathcal{Y}(i_1,\dots,i_d)$, applying \eqref{multiply} to \eqref{yyyy1} obtains \eqref{yyyy2}, and we derive \eqref{yyyy3} from \eqref{yyyy2} by the independence of TT-cores $\{\mathcal{R}_k\}^d_{k=1}$.
			
			Considering the $k$-th term of the right hand side of \eqref{yyyy3}, for $k=1,\dots,d$, we obtain that
			\begin{align}
		 	\mathbb{E}\Big[&\mathcal{Y}_k(i_k)\otimes\mathcal{Y}_k(i_k)\otimes\mathcal{Y}_k(i_k)\otimes\mathcal{Y}_k(i_k)\Big] \nonumber\\
		 	=&\mathbb{E}\Bigg[\Big[\sum_{j_k}\mathcal{R}_k(i_k,j_k)\otimes\mathcal{X}_k(j_k)\Big]\otimes\Big[\sum_{j_k}\mathcal{R}_k(i_k,j_k)\otimes\mathcal{X}_k(j_k)\Big] \nonumber\\
			&\otimes\Big[\sum_{j_k}\mathcal{R}_k(i_k,j_k)\otimes\mathcal{X}_k(j_k)\Big]\otimes\Big[\sum_{j_k}\mathcal{R}_k(i_k,j_k)\otimes\mathcal{X}_k(j_k)\Big]\Bigg] \label{yyyk1}\\
		=&\mathbb{E}\Bigg[\Big[\sum_{j_k}\mathcal{R}_k(i_k,j_k)\mathcal{X}_k(j_k)\Big]\otimes\Big[\sum_{j_k}\mathcal{R}_k(i_k,j_k)\mathcal{X}_k(j_k)\Big] \nonumber\\
		&\otimes\Big[\sum_{j_k}\mathcal{R}_k(i_k,j_k)\mathcal{X}_k(j_k)\Big]\otimes\Big[\sum_{j_k}\mathcal{R}_k(i_k,j_k)\mathcal{X}_k(j_k)\Big]\Bigg] \label{yyyk2}\\
			=&\mathbb{E}\Big[\sum_{j_k}\mathcal{R}^4_k(i_k,j_k)\mathcal{X}_k(j_k)\otimes\mathcal{X}_k(j_k)\otimes\mathcal{X}_k(j_k)\otimes\mathcal{X}_k(j_k)\Big] \nonumber\\
				&+\mathbb{E}\Big[\sum_{j_k\neq j'_k}\mathcal{R}^2_k(i_k,j_k)\mathcal{R}^2_k(i_k,j'_k)\mathcal{X}_k(j_k)\otimes\mathcal{X}_k(j_k)\otimes\mathcal{X}_k(j'_k)\otimes\mathcal{X}_k(j'_k)\Big] \nonumber\\
				&+\mathbb{E}\Big[\sum_{j_k\neq j'_k}\mathcal{R}^2_k(i_k,j_k)\mathcal{R}^2_k(i_k,j'_k)\mathcal{X}_k(j_k)\otimes\mathcal{X}_k(j'_k)\otimes\mathcal{X}_k(j_k)\otimes\mathcal{X}_k(j'_k)\Big] \nonumber\\
				&+\mathbb{E}\Big[\sum_{j_k\neq j'_k}\mathcal{R}^2_k(i_k,j_k)\mathcal{R}^2_k(i_k,j'_k)\mathcal{X}_k(j_k)\otimes\mathcal{X}_k(j'_k)\otimes\mathcal{X}_k(j'_k)\otimes\mathcal{X}_k(j_k)\Big] \label{yyyk3}\\
			=&\Delta\sum_{j_k}\mathcal{X}_k(j_k)\otimes\mathcal{X}_k(j_k)\otimes\mathcal{X}_k(j_k)\otimes\mathcal{X}_k(j_k)
				+\sum_{j_k\neq j'_k}\mathcal{X}_k(j_k)\otimes\mathcal{X}_k(j_k)\otimes\mathcal{X}_k(j'_k)\otimes\mathcal{X}_k(j'_k) \nonumber\\
				&+\sum_{j_k\neq j'_k}\mathcal{X}_k(j_k)\otimes\mathcal{X}_k(j'_k)\otimes\mathcal{X}_k(j_k)\otimes\mathcal{X}_k(j'_k)
				+\sum_{j_k\neq j'_k}\mathcal{X}_k(j_k)\otimes\mathcal{X}_k(j'_k)\otimes\mathcal{X}_k(j'_k)\otimes\mathcal{X}_k(j_k), \label{yyyk4}
			\end{align}
			where we infer \eqref{yyyk2} from \eqref{yyyk1} by scalar property of $\mathcal{R}_k(i_k,j_k)$, \eqref{yyyk3} is obtained by \eqref{distributive} and the independence of TT-cores $\{\mathcal{R}_k\}^d_{k=1}$, and denoting the fourth moment $\Delta:=\mathbb{E}\Big[\mathcal{R}^4_k(i_k,j_k)\Big]$, we deduce \eqref{yyyk4} by the assumption $\mathbb{E}\Big[\mathcal{R}^2_k (i_k,j_k)\Big]=1$, for $k=1,\dots,d$.
			
			Substituting \eqref{yyyk4} into \eqref{yyyy3}, it implies that 
			\begin{align}
			   \mathbb{E}&\Big[\mathcal{Y}^4(i_1,\dots,i_d)\Big] \nonumber\\
			   =&\Big[\Delta\sum_{j_1}\mathcal{X}_1(j_1)\otimes\mathcal{X}_1(j_1)\otimes\mathcal{X}_1(j_1)\otimes\mathcal{X}_1(j_1)+\sum_{j_1\neq j'_1}\mathcal{X}_1(j_1)\otimes\mathcal{X}_1(j_1)\otimes\mathcal{X}_1(j'_1)\otimes\mathcal{X}_1(j'_1) \nonumber\\
			    &+\sum_{j_1\neq j'_1}\mathcal{X}_1(j_1)\otimes\mathcal{X}_1(j'_1)\otimes\mathcal{X}_1(j_1)\otimes\mathcal{X}_1(j'_1)+\sum_{j_1\neq j'_1}\mathcal{X}_1(j_1)\otimes\mathcal{X}_1(j'_1)\otimes\mathcal{X}_1(j'_1)\otimes\mathcal{X}_1(j_1)\Big]  \nonumber \\
			    &\cdots\Big[\Delta\sum_{j_d}\mathcal{X}_d(j_d)\otimes\mathcal{X}_d(j_d)\otimes\mathcal{X}_d(j_d)\otimes\mathcal{X}_d(j_d)+\sum_{j_d\neq j'_d}\mathcal{X}_d(j_d)\otimes\mathcal{X}_d(j_d)\otimes\mathcal{X}_d(j'_d)\otimes\mathcal{X}_d(j'_d) \nonumber\\
			    &+\sum_{j_d\neq j'_d}\mathcal{X}_d(j_d)\otimes\mathcal{X}_d(j'_d)\otimes\mathcal{X}_d(j_d)\otimes\mathcal{X}_d(j'_d)+\sum_{j_d\neq j'_d}\mathcal{X}_d(j_d)\otimes\mathcal{X}_d(j'_d)\otimes\mathcal{X}_d(j'_d)\otimes\mathcal{X}_d(j_d)\Big]  \nonumber\\
			    \leq& \Delta^d \sum_{j_1,\dots,j_d}\Bigg[\Big[\mathcal{X}_1(j_1)\otimes\mathcal{X}_1(j_1)\otimes\mathcal{X}_1(j_1)\otimes\mathcal{X}_1(j_1)\Big]\cdots\Big[\mathcal{X}_d(j_d)\otimes\mathcal{X}_d(j_d)\otimes\mathcal{X}_d(j_d)\otimes\mathcal{X}_d(j_d)\Big]\Bigg]  \nonumber\\
			    &+\Delta^{d-1}C_d^1\underset{k}{\max}\Bigg[\sum_{j_1,..,j_k\neq j'_k,\dots, j_d}\Big[\mathcal{X}_1(j_1)\otimes\mathcal{X}_1(j_1)\otimes\mathcal{X}_1(j_1)\otimes\mathcal{X}_1(j_1)\Big]\cdots \nonumber\\
			    &\Big[\mathcal{X}_k(j_k)\otimes\mathcal{X}_k(j_k)\otimes\mathcal{X}_k(j'_k)\otimes\mathcal{X}_k(j'_k)\Big]
			    \cdots \Big[\mathcal{X}_d(j_d)\otimes\mathcal{X}_d(j_d)\otimes\mathcal{X}_d(j_d)\otimes\mathcal{X}_d(j_d)\Big]\Bigg] \nonumber\\
			    &+\Delta^{d-1}C_d^1\underset{k}{\max}\Bigg[\sum_{j_1,..,j_k\neq j'_k,\dots, j_d}\Big[\mathcal{X}_1(j_1)\otimes\mathcal{X}_1(j_1)\otimes\mathcal{X}_1(j_1)\otimes\mathcal{X}_1(j_1)\Big]\cdots  \nonumber\\
			    &\Big[\mathcal{X}_k(j_k)\otimes\mathcal{X}_k(j'_k)\otimes\mathcal{X}_k(j_k)\otimes\mathcal{X}_k(j'_k)\Big]\cdots \Big[\mathcal{X}_d(j_d)\otimes\mathcal{X}_d(j_d)\otimes\mathcal{X}_d(j_d)\otimes\mathcal{X}_d(j_d)\Big]\Bigg]  \nonumber\\
			    &+\Delta^{d-1}C_d^1\underset{k}{\max}\Bigg[\sum_{j_1,..,j_k\neq j'_k,\dots, j_d}\Big[\mathcal{X}_1(j_1)\otimes\mathcal{X}_1(j_1)\otimes\mathcal{X}_1(j_1)\otimes\mathcal{X}_1(j_1)\Big]\cdots  \nonumber\\
			    &\Big[\mathcal{X}_k(j_k)\otimes\mathcal{X}_k(j'_k)\otimes\mathcal{X}_k(j'_k)\otimes\mathcal{X}_k(j_k)\Big]
			    \cdots \Big[\mathcal{X}_d(j_d)\otimes\mathcal{X}_d(j_d)\otimes\mathcal{X}_d(j_d)\otimes\mathcal{X}_d(j_d)\Big]\Bigg]+\cdots \label{computing_delta}\\
			    \leq& \Delta^d\sum_{j_1,\dots,j_d}\mathcal{X}^4(j_1,\dots,j_d)+ 3\Delta^{d-1}C_d^1\underset{k}{\max}\Bigg[\sum_{j_1,..,j_k\neq j'_k,\dots, j_d} \mathcal{X}(j_1,\dots,j_k,\dots,j_d)^2\mathcal{X}(j_1,\dots,j'_k,\dots,j_d)^2 \Bigg] + \cdots \label{delta_result}\\
			    \leq& \Delta^d{\Arrowvert \mathbf{x} \Arrowvert}^4_4+3(n-1)\Delta^{d-1}C^1_d N\mathcal{M}^4+3^2(n-1)^2\Delta^{d-2}C^2_d N\mathcal{M}^4+\cdots+3^d(n-1)^d N\mathcal{M}^4 \nonumber\\
				\leq&\Big(\Delta+3(n-1)\Big)^d N\mathcal{M}^4, \label{yiresult}
			    \end{align}
			    where denoting $\mathcal{M}:=\max_{i=1,\dots,N} \ \lvert\mathbf{x}(i)\rvert,\, n=\max\{n_1,n_2,\dots,n_d\}$, we derive \eqref{delta_result} from \eqref{computing_delta} by \eqref{multiply}.

Similarly, the second term $\mathbb{E}\Big[\mathbf{y}^2(i)\mathbf{y}^2(j)\Big]$ of the right hand side of \eqref{square_right}, for $i\neq j,\,\nu(i)=(i_1,i_2,\dots,i_d)\neq \nu(j)=(i'_1,i'_2,\dots,i'_d)$, is obtained by 
			\begin{align}
			  \mathbb{E}&\Big[\mathbf{y}^2(i)\mathbf{y}^2(j)\Big] \nonumber\\
			  =&\mathbb{E}\Big[\mathcal{Y}_1(i_1)\otimes\mathcal{Y}_1(i_1)\otimes\mathcal{Y}_1(i'_1)\otimes\mathcal{Y}_1(i'_1)\Big]\cdots\mathbb{E}\Big[\mathcal{Y}_d(i_d)\otimes\mathcal{Y}_d(i_d)\otimes\mathcal{Y}_d(i'_d)\otimes\mathcal{Y}_d(i'_d)\Big]. \label{yiyj}
			 \end{align}
If $i_k\neq i'_k$, for $k=1,\dots,d$, then the $k$-th term of the right hand side of \eqref{yiyj} is computed by
\begin{align}
			\mathbb{E}&\Big[\mathcal{Y}_k(i_k)\otimes\mathcal{Y}_k(i_k)\otimes\mathcal{Y}_k(i'_k)\otimes\mathcal{Y}_k(i'_k)\Big] \nonumber\\     
			 =&\mathbb{E}\Bigg[\Big[\sum_{j_k}\mathcal{R}_k(i_k,j_k)\mathcal{X}_k(j_k)\Big]\otimes\Big[\sum_{j_k}\mathcal{R}_k(i_k,j_k)\mathcal{X}_k(j_k)\Big] \nonumber\\
				&\otimes\Big[\sum_{j_k}\mathcal{R}_k(i'_k,j_k)\mathcal{X}_k(j_k)\Big]\otimes\Big[\sum_{j_k}\mathcal{R}_k(i'_k,j_k)\mathcal{X}_k(j_k)\Big]\Bigg]\\
				=&\mathbb{E}\Big[\sum_{j_k}\mathcal{R}^2_k(i_k,j_k)\mathcal{R}^2_k(i'_k,j_k)\mathcal{X}_k(j_k)\otimes\mathcal{X}_k(j_k)\otimes\mathcal{X}_k(j_k)\otimes\mathcal{X}_k(j_k)\Big]\nonumber\\
				&+\mathbb{E}\Big[\sum_{j_k\neq j'_k}\mathcal{R}^2_k(i_k,j_k)\mathcal{R}^2_k(i'_k,j'_k)\mathcal{X}_k(j_k)\otimes\mathcal{X}_k(j_k)\otimes\mathcal{X}_k(j'_k)\otimes\mathcal{X}_k(j'_k)\Big]\\
				=&\sum_{j_k}\mathcal{X}_k(j_k)\otimes\mathcal{X}_k(j_k)\otimes\mathcal{X}_k(j_k)\otimes\mathcal{X}_k(j_k)
				+\sum_{j_k\neq j'_k}\mathcal{X}_k(j_k)\otimes\mathcal{X}_k(j_k)\otimes\mathcal{X}_k(j'_k)\otimes\mathcal{X}_k(j'_k). \label{ykresult}
\end{align}
Supposing that $i_1=i'_1,\dots,i_k\neq i'_k,\dots,i_d=i'_d$ and substituting \eqref{yyyk4} and \eqref{ykresult} into \eqref{yiyj}, we obtain
\begin{align}
    \mathbb{E}&\Big[\mathbf{y}^2(i)\mathbf{y}^2(j)\Big] \nonumber\\
    =&	\mathbb{E}\Big[\mathcal{Y}_1(i_1)\otimes\mathcal{Y}_1(i_1)\otimes\mathcal{Y}_1(i_1)\otimes\mathcal{Y}_1(i_1)\Big]\cdots \mathbb{E}\Big[\mathcal{Y}_k(i_k)\otimes\mathcal{Y}_k(i_k)\otimes\mathcal{Y}_k(i'_k)\otimes\mathcal{Y}_k(i'_k)\Big]\cdots \nonumber\\
    &\mathbb{E}\Big[\mathcal{Y}_d(i_d)\otimes\mathcal{Y}_d(i_d)\otimes\mathcal{Y}_d(i_d)\otimes\mathcal{Y}_d(i_d)\Big] \nonumber\\
    =&\Big[\Delta\sum_{j_1}\mathcal{X}_1(j_1)\otimes\mathcal{X}_1(j_1)\otimes\mathcal{X}_1(j_1)\otimes\mathcal{X}_1(j_1)
				+\sum_{j_1\neq j'_1}\mathcal{X}_1(j_1)\otimes\mathcal{X}_1(j_1)\otimes\mathcal{X}_1(j'_1)\otimes\mathcal{X}_1(j'_1) \nonumber\\
				&+\sum_{j_1\neq j'_1}\mathcal{X}_1(j_1)\otimes\mathcal{X}_1(j'_1)\otimes\mathcal{X}_1(j_1)\otimes\mathcal{X}_1(j'_1)
				+\sum_{j_1\neq j'_1}\mathcal{X}_1(j_1)\otimes\mathcal{X}_1(j'_1)\otimes\mathcal{X}_1(j'_1)\otimes\mathcal{X}_1(j_1)\Big]\nonumber\\
                & \cdots\Big[\sum_{j_k}\mathcal{X}_k(j_k)\otimes\mathcal{X}_k(j_k)\otimes\mathcal{X}_k(j_k)\otimes\mathcal{X}_k(j_k)
				+\sum_{j_k\neq j'_k}\mathcal{X}_k(j_k)\otimes\mathcal{X}_k(j_k)\otimes\mathcal{X}_k(j'_k)\otimes\mathcal{X}_k(j'_k)\Big]\nonumber\\
				&\cdots\Big[\Delta\sum_{j_d}\mathcal{X}_d(j_d)\otimes\mathcal{X}_d(j_d)\otimes\mathcal{X}_d(j_d)\otimes\mathcal{X}_d(j_d)
				+\sum_{j_d\neq j'_d}\mathcal{X}_d(j_d)\otimes\mathcal{X}_d(j_d)\otimes\mathcal{X}_d(j'_d)\otimes\mathcal{X}_d(j'_d) \nonumber\\
				&+\sum_{j_d\neq j'_d}\mathcal{X}_d(j_d)\otimes\mathcal{X}_d(j'_d)\otimes\mathcal{X}_d(j_d)\otimes\mathcal{X}_d(j'_d)
				+\sum_{j_d\neq j'_d}\mathcal{X}_d(j_d)\otimes\mathcal{X}_d(j'_d)\otimes\mathcal{X}_d(j'_d)\otimes\mathcal{X}_d(j_d)\Big] \nonumber\\
				\leq& n(\Delta+3(n-1))^{d-1} N\mathcal{M}^4. \label{yiyj1}
\end{align}
Similarly, if for $k\in S \subseteq \{1,\dots,d\},\,\lvert S\rvert=l$, $i_k\neq i'_k$, and for $k\in \overline{S}$, $i_k=i'_k$, then
\begin{align}
    \mathbb{E}\Big[\mathbf{y}^2(i)\mathbf{y}^2(j)\Big]\leq n^l(\Delta+3(n-1))^{d-l} N\mathcal{M}^4. \label{yiyjl}
\end{align}
Hence, combining \eqref{yiyj1} and \eqref{yiyjl} gives
\begin{align}
    \sum_{i\neq j}\mathbb{E}\Big[\mathbf{y}^2(i)\mathbf{y}^2(j)\Big]
    \leq& M\Big[C^1_d(m-1)n(\Delta+3(n-1))^{d-1}+\cdots+C^l_d(m-1)^ln^l(\Delta+3(n-1))^{(d-l)} \nonumber\\
    &+\cdots+C^d_d(m-1)^d n^d\Big]N\mathcal{M}^4, \label{totalyiyj}
\end{align}
where $m=\max\{m_1,m_2,\dots,m_d\}$.\\
Therefore, using \eqref{yiresult} and \eqref{totalyiyj} deduces
			\begin{align}
		 \mathbb{E}\Big[{\Arrowvert\mathbf{y}\Arrowvert}^{4}_2\Big]\leq& M\Big[(\Delta+3(n-1))^d+C^1_d(m-1)n(\Delta+3(n-1))^{d-1}+\cdots+C^d_d(m-1)^d n^d\Big]N\mathcal{M}^4 \nonumber\\
		 &=M\Big((m-1)n+\Delta+3(n-1)\Big)^d N \mathcal{M}^4 \nonumber\\
			    &=M\Big(\Delta+n(m+2)-3\Big)^d N \mathcal{M}^4. \label{y24}
			\end{align}
			In summary, substituting \eqref{y24} into \eqref{var_right} implies
			\begin{align}
				\text{Var}\Big({\Arrowvert f_{TTRP}(\mathbf{x})\Arrowvert}^{2}_2\Big)\leq&\frac{M\Big(\Delta+n(m+2)-3\Big)^d N \mathcal{M}^4}{M^2}-{\Arrowvert \mathbf{x} \Arrowvert}^4_2 \nonumber\\
				\leq& \frac{1}{M}\Big(\Delta+n(m+2)-3\Big)^d N\mathcal{M}^4-{\Arrowvert \mathbf{x} \Arrowvert}^4_2. \label{upper_bound}
			\end{align}
	\end{proof}
	{\lr One can see that the bound of the variance \eqref{upper_bound} is reduced as $M$ increases, which is expected. When $M=m^d$ and $N=n^d$, we have 
	\begin{align}
	    \text{Var}\Big({\Arrowvert f_{TTRP}(\mathbf{x})\Arrowvert}^{2}_2\Big)&\leq \Big(\frac{\Delta+2n-3}{m}+n\Big)^d N\mathcal{M}^4-{\Arrowvert \mathbf{x} \Arrowvert}^4_2.
	    \label{upper_d}
	\end{align}
	As $m$ increases, the upper bound in \eqref{upper_d} tends to $(N^2 \mathcal{M}^4-{\Arrowvert \mathbf{x} \Arrowvert}^4_2)\geq0$, and this upper bound vanishes as $M$ increases if and only if $\mathbf{x}(1)=\mathbf{x}(2)=\dots=\mathbf{x}(N)$.} 
	 Also, {\lr the upper bound \eqref{upper_bound}} is affected by the fourth moment $\Delta=\mathbb{E}\Big[\mathcal{R}^4_k(i_k,j_k)\Big]=\text{Var}\Big(\mathcal{R}^2_k(i_k,j_k)\Big)+\Big[\mathbb{E}[\mathcal{R}^2_k(i_k,j_k)]\Big]^2$. To keep the expected isometry, we need $\mathbb{E}[\mathcal{R}^2_k(i_k,j_k)]=1$.
	 Note that when the TT-cores follow the Rademacher distribution i.e.,\,$\text{Var}\Big(\mathcal{R}^2_k(i_k,j_k)\Big)=0$, the fourth moment $\Delta$ in \eqref{upper_bound} achieves the minimum. So, the Rademacher distribution is an optimal choice for generating the TT-cores, and we set the Rademacher distribution to be our default choice for constructing TTRP (Definition \ref{defTTRP}). 
	 \begin{proposition}\label{hyper}
	 (Hypercontractivity \cite{schudy2012concentration}) Consider a degree $q$ polynomial $f(Y)=$ $f\left(Y_{1}, \ldots, Y_{n}\right)$ of independent centered Gaussian or Rademacher random variables $Y_{1}, \ldots, Y_{n} .$ Then for any $\lambda>0$
\begin{equation*}
    \mathbb{P}\left(\left\lvert f(Y)-\mathbb{E}\left[f(Y)\right]\right\rvert \geq \lambda\right) \leq e^{2} \cdot \exp{\left[-\left(\frac{\lambda^2}{K\cdot \text{Var}[f(Y)] }\right)^{\frac{1}{q}}\right]},
\end{equation*}
where $\operatorname{Var}([f(Y)])$ is the variance of the random variable $f(Y)$ and $K>0$ is an absolute constant.
	 \end{proposition}
Proposition \ref{hyper} extends the Hanson-Wright inequality whose proof can be found in \cite{schudy2012concentration}. 
	 
\begin{proposition}\label{con}
Let $f_{TTRP}: \dsR^N \mapsto \dsR^{M}$ be the tensor train random projection defined by \eqref{eq_ttfpfull}. Suppose that for $i=1, \ldots, d$, all entries of TT-cores $\mathcal{R}_i$ are independent standard Gaussian or Rademacher random variables, with the same fourth moment $\Delta$ and $\mathcal{M}:=\max_{i=1,\dots,N} \ \lvert \mathbf{x}(i)\rvert,\,m=\max\{m_1,m_2,\dots,m_d\},\, n=\max\{n_1,n_2,\dots,n_d\}$. For any $\mathbf{x} \in \dsR^{N}$, there exist absolute constants $C$ and $K>0$ such that the following claim holds
\begin{equation}
    \mathbb{P} \left ( \left\lvert \norm{f_{TTRP}(\mathbf{x})}{2}^2 - \norm{\mathbf{x}}{2}^2 \right\rvert \geq \varepsilon \norm{\mathbf{x}}{2}^2 \right ) \leq  C \exp{\left[-\left(\frac{M\cdot\varepsilon^2 }{K\cdot\left[\left(\Delta+n(m+2)-3\right)^d N-M\right]}\right)^\frac{1}{2d}\right]}. \label{prop2_ineq}
\end{equation}
\end{proposition}
\begin{proof}
According to Theorem \ref{lemma_mean}, $	\mathbb{E}{\Arrowvert f_{TTRP}(\mathbf{x})\Arrowvert}^{2}_2={\Arrowvert \mathbf{x} \Arrowvert}^2_2$. Since ${\Arrowvert f_{TTRP}(\mathbf{x})\Arrowvert}^{2}_2$ is a polynomial of degree $2d$ of independent standard Gaussian or Radamecher random variables, which are the entries of TT-cores $\mathcal{R}_i$, for $i=1,\dots,d$, we apply Proposition \ref{hyper} and Theorem \ref{lemma_var} to obtain
{\lr
\begin{align*}
		\mathbb{P} \left ( \left\lvert \norm{f_{TTRP}(\mathbf{x})}{2}^2 - \norm{\mathbf{x}}{2}^2 \right\rvert \geq \varepsilon \norm{\mathbf{x}}{2}^2 \right ) &\leq e^2\cdot \exp{\left[-{\left(\frac{\varepsilon^2 \norm{\mathbf{x}}{2}^4}{K\cdot \text{Var}\left(\norm{f_{TTRP}(\mathbf{x})}{2}^2\right)}\right)}^\frac{1}{2d}\right]}\\
		&\leq e^2\cdot \exp{\left[-\left(\frac{\varepsilon^2 }{K\cdot\left[\frac{1}{M}\left(\Delta+n(m+2)-3\right)^d N\frac{\mathcal{M}^4}{\norm{\mathbf{x}}{2}^4}-1\right]}\right)^\frac{1}{2d}\right]}\\
		&\leq e^2\cdot \exp{\left[-\left(\frac{M\cdot\varepsilon^2 }{K\cdot\left[\left(\Delta+n(m+2)-3\right)^d N-M\right]}\right)^\frac{1}{2d}\right]}\\
		&\leq C \exp{\left[-\left(\frac{M\cdot\varepsilon^2 }{K\cdot\left[\left(\Delta+n(m+2)-3\right)^d N-M\right]}\right)^\frac{1}{2d}\right] },
\end{align*}
where  $\mathcal{M}=\max_{i=1,\dots,N} \ \lvert \mathbf{x}(i)\rvert$ and then $\frac{\mathcal{M}^4}{\norm{\mathbf{x}}{2}^4}\leq 1$.
}
\end{proof}
We note that the upper bound in the concentration inequality \eqref{prop2_ineq} is not tight, as it involves the dimensionality of datasets ($N$).
To give a tight bound independent of the dimensionality of datasets 
for the corresponding concentration inequality is our future work.

The procedure of TTRP is summarized in Algorithm \ref{alg_1}. 
For the input of this algorithm, the TT-ranks of $\mathcal{R}$ (the tensorized version of  the projection matrix $\mathbf{R}$ in  \eqref{eq_ttfpfull}) are set to  one, and from our above analysis, we generate entries of the corresponding TT-cores $\{\mathcal{R}_{k}\}^d_{k=1}$ through the Rademacher distribution. 
For a given data point $\mathbf{x}$ in the TT-format, Algorithm \ref{alg_1} gives the TT-cores of the corresponding output, and  each element of $f_{TTRP}(\mathbf{x})$ in \eqref{eq_ttfpfull} can be represented as:
 $$f_{TTRP}(\mathbf{x})(i)=f_{TTRP}(\mathbf{x})(\nu(i))=f_{TTRP}(\mathbf{x})(i_1,\dots,i_d)=\frac{1}{\sqrt{M}}\mathcal{Y}_1(i_1)\cdots\mathcal{Y}_d(i_d),$$
 where $\nu$ is a bijection from $\Ne $ to $\Ne^{d}$.
\begin{algorithm}[H]
	\caption{Tensor train random projection}
	\label{alg_1}
	\begin{algorithmic}[1]
		\Require TT-cores $\mathcal{R}_{k}\left(i_{k}, j_{k}\right)$ of  $\mathbf{R}$, and TT-cores $\mathcal{X}_{k}$ of  $\mathbf{x}$, for $k=1,\dots,d$.
		\For {$k = 1:d$}
		\For {$i_k=1:m_k$}
		\State Compute $\mathcal{Y}_{k}\left(i_{k}\right)=\sum_{j_{k}=1}^{n_k}\Big(\mathcal{R}_{k}\left(i_{k}, j_{k}\right) \otimes \mathcal{X}_{k}\left(j_{k}\right)\Big)$. $\qquad \triangleright \ O(n\hat{r}^2)$ by \eqref{ycore}
		\EndFor
		\EndFor
		\Ensure TT-cores $\frac{1}{\sqrt{M}}\mathcal{Y}_1$, $\mathcal{Y}_2,\dots,$ $\mathcal{Y}_d$.
	\end{algorithmic}
\end{algorithm}



\section{Numerical experiments}\label{Experm}
We demonstrate the efficiency of TTRP using synthetic datasets and the MNIST dataset \cite{lecun2010mnist}. 
The quality of isometry is a key factor to assess the performance of random
projection methods, 
which in our numerical studies is
estimated by the ratio of the pairwise distance
	\begin{equation}\label{ratio}
	\frac{2}{n_0(n_0-1)}\sum_{n_0\geq i > j}\frac{{\Arrowvert  f_{TTRP}(\mathbf{x}^{(i)}) - f_{TTRP}(\mathbf{x}^{(j)}) \Arrowvert}_2}{{\Arrowvert \mathbf{x}^{(i)}-\mathbf{x}^{(j)}\Arrowvert}_2},
	\end{equation}
	where $n_0$ is the number of data points. Since the output of our TTRP procedure (see Algorithm \ref{alg_1}) is in the TT-format, it is efficient to  apply TT-format operations to compute the pairwise distance of \eqref{ratio} through
	Algorithm \ref{dot}. In order to obtain the average performance of isometry, we repeat numerical experiments 100 times (different realizations for TT-cores)
	and	estimate the mean and the variance for the ratio of the pairwise distance using these samples.
The rest of this section is organized as follows. First, 
through a synthetic dataset, 
the effect of different TT-ranks of the tensorized version $\mathcal{R}$ of $\mathbf{R}$ in \eqref{eq_ttfpfull} is shown, 
which leads to our motivation of setting the TT-ranks to be one.
After that, we focus on the situation
with TT-ranks equal to one, and test the effect of different TT-cores. Finally, based on both high-dimensional synthetic and  MNIST datasets, our TTRP are compared with  related projection methods, including Gaussian TRP \cite{sun2018tensor}, Very Sparse RP \cite{li2006very} and Gaussian RP \cite{achlioptas2001database}.

 
\subsection{Effect of different TT-ranks}\label{section_rank}

In Definition \ref{defTTRP}, we set the TT-ranks to be one.
To explain our motivation of this settting, we investigate the effect of different TT-ranks---we herein consider the situation that the TT-ranks take $r_0=r_d=1,\, r_k=r,\,k=2,\dots,d-1$,
where the rank $r\in \{1,2,\ldots\}$, 
and we keep other settings in Definition \ref{defTTRP} unchanged.
For comparison, two different distributions are considered to generate the TT-cores in this part---the Rademacher distribution (our default optimal choice) and the Gaussian distribution, and the corresponding tensor train projection is denoted by  rank-$r$ TTRP and Gaussian TT (studied in detail in \cite{rakhshan2020tensorized}) respectively. 
For rank-$r$ TTRP, 
the entries of TT-cores  $\mathcal{R}_1(i_1,j_1)$ and $\mathcal{R}_d(i_d,j_d)$ are drawn from $1/r^{1/4}$ or $-1/r^{1/4}$ with equal probability, and each element of $\mathcal{R}_k(i_k,j_k),\,k=2,..,d-1$ is uniformly and independently drawn from $1/r^{1/2}$ or $-1/r^{1/2}$.

A synthetic dataset 
with dimension $N=1000$ and size $n_0=10$ are  generated,
where each entry of vectors (each vector is a sample in the synthetic dataset) is independently generated through $\mathcal{N}(0,1)$.
In this test problem, we set the reduced dimension to be 
$M=24$, and the  dimensions of the corresponding tensor representations are set to $m_1=4,\,m_2=3,\,m_3=2$ and $n_1=n_2=n_3=10$ ($M=m_1m_2m_3$ and $N=n_1n_2n_3$). 
Figure \ref{rttrp} shows the ratio of the pairwise distance of the two  projection methods (computed through \eqref{ratio}).
It can be seen that the estimated mean of ratio of the pairwise distance of rank-$r$ TTRP is typically more close to one than that of Gaussian TT, i.e., rank-$r$ TTRP
has advantages for keeping the pairwise distances. 
Clearly, for a given rank in Figure \ref{rttrp}, the estimated variance of the pairwise distance of rank-$r$ TTRP is 
 smaller 
than that of Gaussian TT. 
Moreover, focusing on rank-$r$ TTRP,  the results of both the mean and the variance are not significantly different for different TT-ranks. In order to reduce the storage, we only focus on the rank-one case (as in Definition \ref{defTTRP}) in the rest of this paper.
\begin{figure}
    \centering
	\subfloat[][Mean for the ratio of the pairwise distance ]{\includegraphics[width=.48\textwidth]{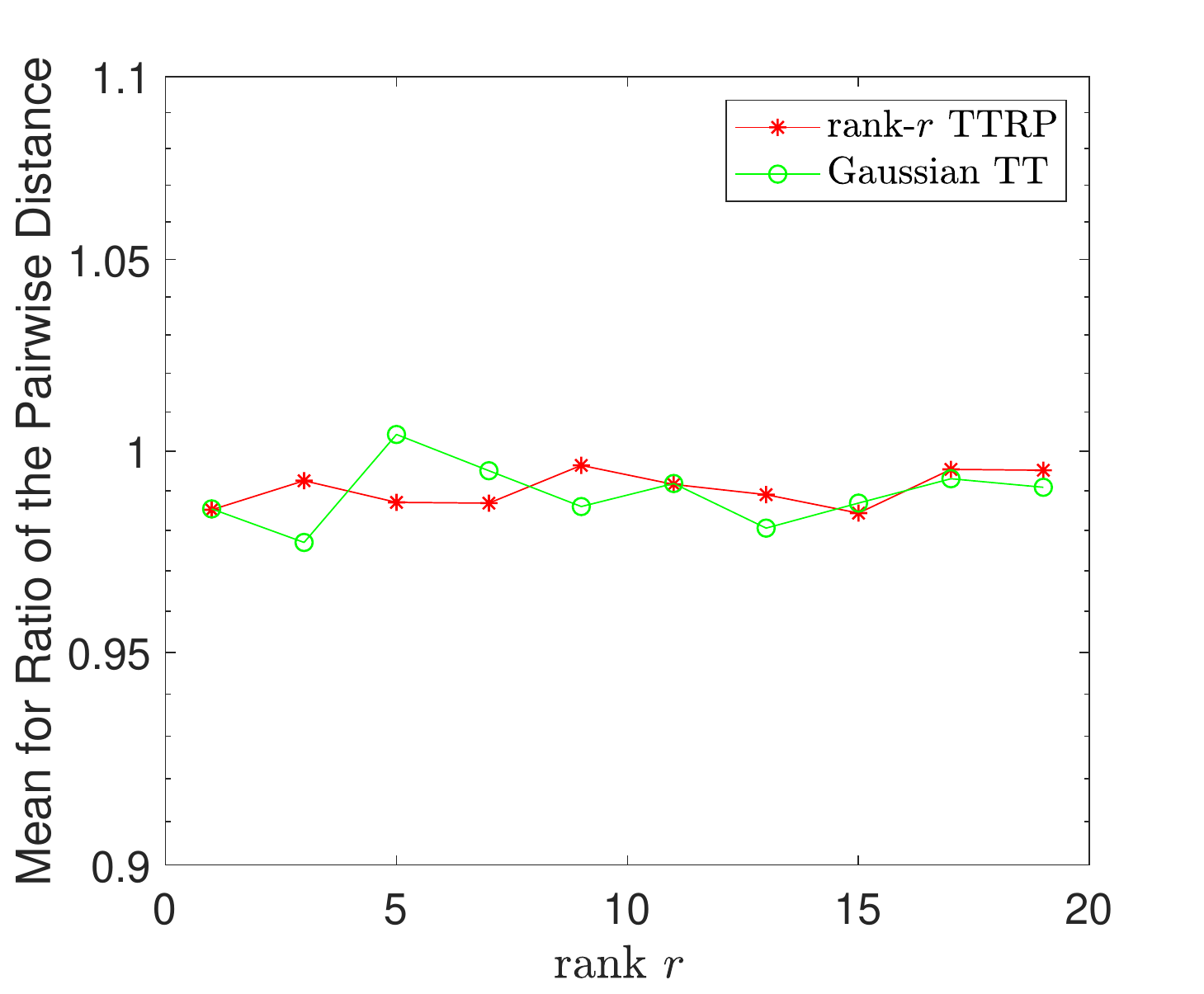}}\quad
	\subfloat[][Variance for the ratio of the pairwise distance]{\includegraphics[width=.48\textwidth]{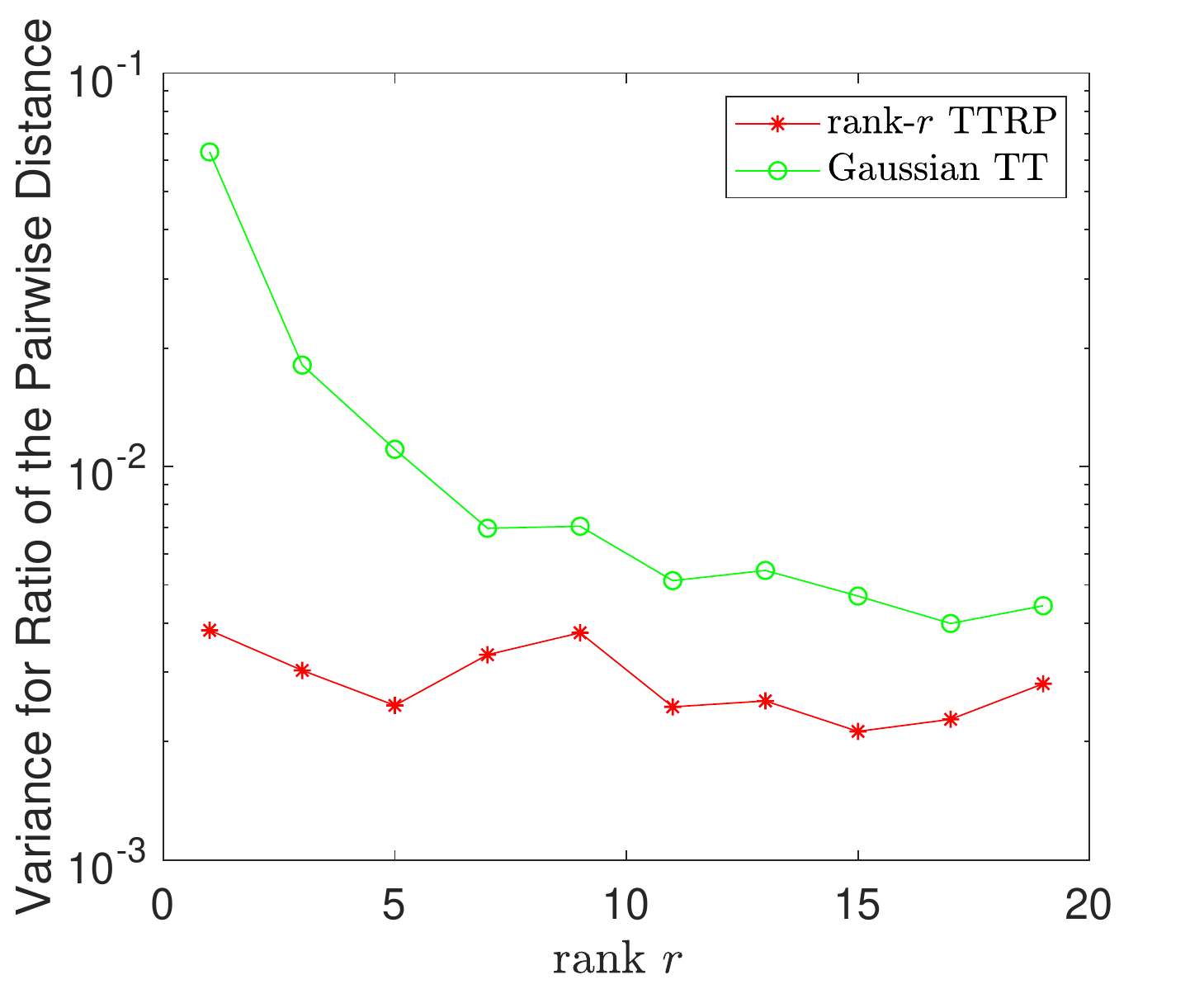}}
    \caption{Effect of different ranks based on synthetic data ($M=24,\,N=1000,\,m_1=4,\,m_2=3,\,m_3=2,\,n_1=n_2=n_3=10$).}
    \label{rttrp}
\end{figure}

\subsection{Effect of different TT-cores}\label{cores}
A synthetic dataset
is tested to assess the effect of different distributions for TT-cores, which consists of independent vectors $\mathbf{x}^{(1)},\dots,\mathbf{x}^{(10)},$ with dimension $N=2500$, whose elements are sampled from the standard Gaussian distribution. 
The following three distributions 
are investigated
to construct TTRP (see Definition \ref{defTTRP}), which include the Rademacher distribution (our default choice), the standard Gaussian distribution (studied in \cite{rakhshan2020tensorized}), and the $1/3$-sparse distribution (i.e., $s=3$ in \eqref{sparse_distribution}), while the corresponding projection methods
are denoted by TTRP-RD, TTRP-$\mathcal{N}(0,1)$, and TTRP-$1/3$-sparse, respectively.
For this test problem, three TT-cores are  utilized for $m_1=M/2,\,m_2=2,\,n_3=1$ and $n_1=25,\,n_2=10,\,n_3=10$.
Figure \ref{core} shows that the estimated mean of the ratio of the pairwise distance for TTRP-RD is
very close to one, and the estimated variance of TTRP-RD is 
at least one order of magnitude smaller
than that of TTRP-$\mathcal{N}(0,1)$ and TTRP-$1/3$-sparse.
These results are consist with Theorem \ref{lemma_var}. In the rest of this paper, we focus on our default choice 
for TTRP---the TT-ranks are set to one, and each element of TT-cores is independently sampled through the Rademacher distribution.

\begin{figure}
    \centering
	\subfloat[][Mean for the ratio of the pairwise distance ]{\includegraphics[width=.48\textwidth]{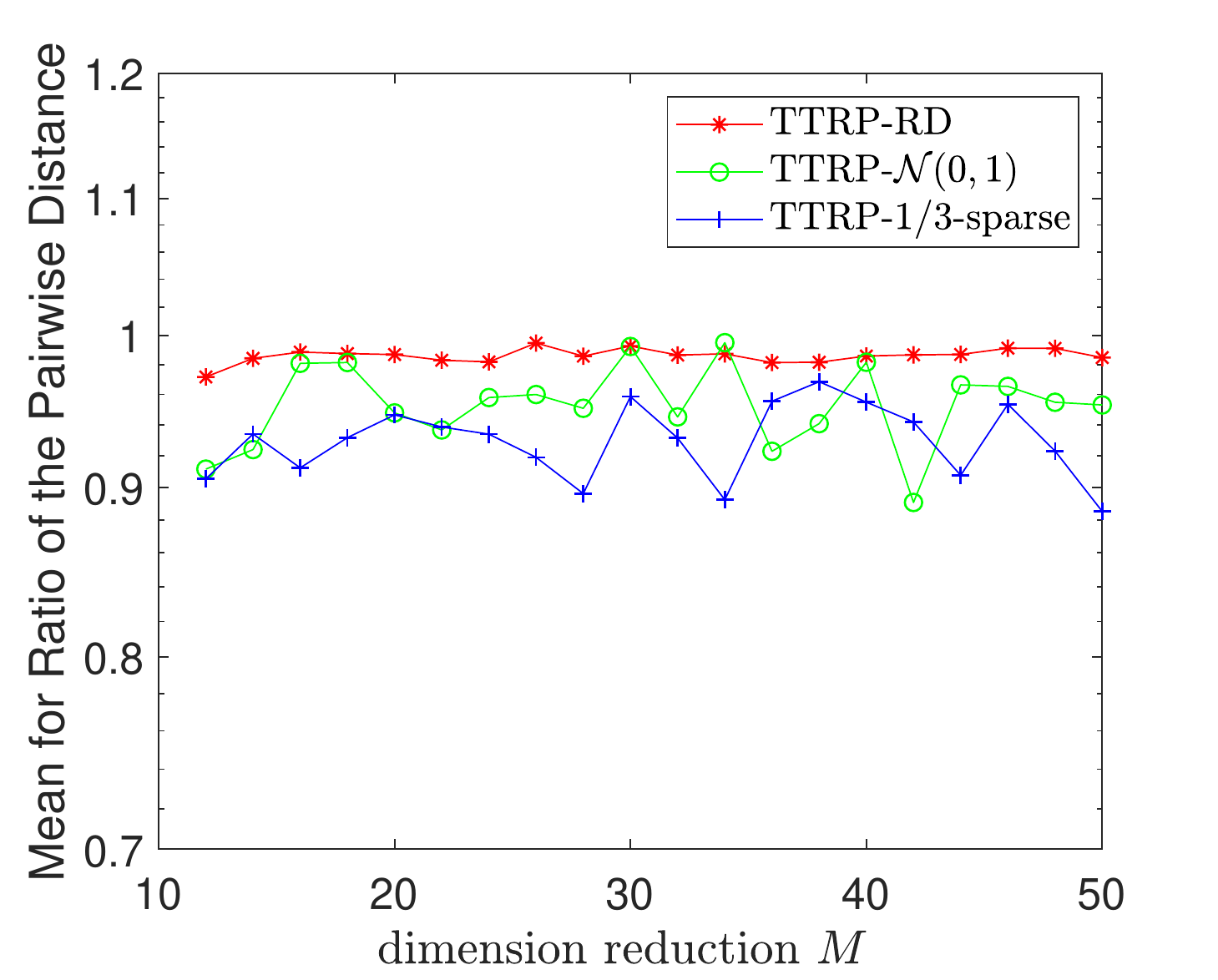}}\quad
	\subfloat[][Variance for the ratio of the pairwise distance]{\includegraphics[width=.48\textwidth]{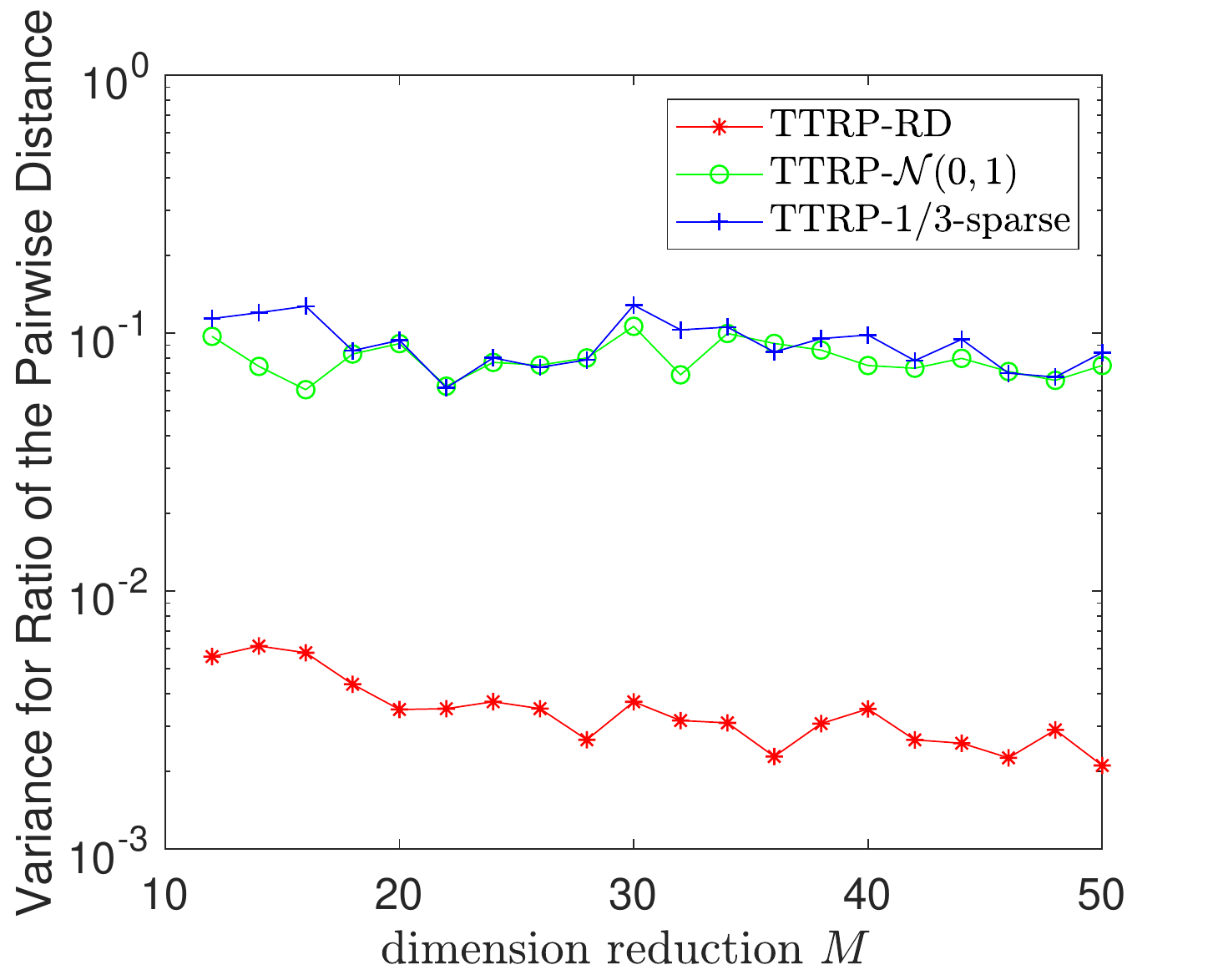}}
    \caption{Three test distributions for TT-cores based on synthetic data ($N=2500$).}
    \label{core}
\end{figure}
	
\subsection{Comparison with Gaussian TRP, Very Sparse RP and Gaussian RP}

The storage of the projection matrix  and
the cost of computing  $\mathbf{Rx}$ (see \eqref{eq_ttfpfull}) 
of our TTRP (TT-ranks equal one),
Gaussian TRP \cite{sun2018tensor}, Very Sparse RP \cite{li2006very} and Gaussian RP \cite{achlioptas2001database}, are shown in Table \ref{storage},  
where $\mathbf{R}\in \mathbb{R}^{M\times N},\,M=\prod_{i=1}^{d} m_i,\,N=\prod_{j=1}^{d} n_j,\,m=\max\{m_1,m_2,\dots,m_d\}$ and $n=\max\{n_1,n_2,\dots,n_d\}$.
Note that the matrix $\mathbf{R}$ in \eqref{eq_ttfpfull} is tensorized in the TT-format,  
and TTRP is efficiently achieved by the matrix-by-vector products in the TT-format (see \eqref{ycore}).
From Table \ref{storage}, it is clear that our TTRP has the smallest storage cost and requires the smallest computational cost for computing $\mathbf{Rx}$.  
	\begin{table}
		\caption{The comparison of the storage and the computational costs.}
		\label{storage}
		\centering
		\begin{tabular}{@{}lllll@{}}
			\toprule
			& Gaussian RP & Very Sparse RP  & Gaussian TRP    & TTRP \\
			\midrule
			Storage cost & $O(MN) $     & $O(M\sqrt{N})$  & $O(dMn)$         & $O(dmn)$   \\
			Computational cost & $O(MN)$  & $O(M\sqrt{N})$  & $O(MN)$          & $O(dmn\hat{r}^2)$ \\
			\bottomrule
		\end{tabular}
	\end{table}
 

Two synthetic datasets with size $n_0=10$ are tested---the dimension of the first one is $N=2500$ and that of the second one is $N=10^4$;
each entry of the samples is independently generated through $\mathcal{N}(0,1)$. 
For TTRP and Gaussian TRP, the  dimensions of tensor representations are set to: for $N=2500$, we set $n_1=25,\,n_2=10,\,n_3=10,\,m_1=M/2,\,m_2=2,\,m_3=1$; for $N=10^4$, we set
 $n_1=n_2=25,\,n_3=n_4=4,\,m_1=M/2,\,m_2=2,\,m_3=1,\,m_4=1$. 
 We again focus on the  ratio of the pairwise distance (putting the outputs of different projection methods into \eqref{ratio}), and estimate the mean and the variance for the ratio of the pairwise distance through repeating numerical experiments 100 times (different realizations for constructing the random projections, e.g., different realizations of the Rademacher distribution for TTRP).

Figure \ref{s2500} shows that the performance of TTRP is very close to that of sparse RP and  Gaussian RP, while the variance for  Gaussian TRP is larger than that for the other three projection methods. Moreover, the variance for TTRP basically reduces as the dimension $M$ increases, which is consistent with Theorem \ref{lemma_var}. To be further, more details are given for the case with $M=24$ and $N=10^4$ in Table \ref{estorage1} and Table \ref{estorage2}, where the value of storage is the number of nonzero entries that need to be stored.  
It turns out that TTRP with fewer storage costs achieves a competitive performance compared with Very Sparse RP and Gaussian RP. In addition, from Table \ref{estorage2}, for $d>2$, the variance of TTRP is clearly smaller than that of Gaussian TRP, and the 
storage cost of TTRP is much smaller than that of Gaussian TRP.
\begin{figure}
	\centering
	\subfloat[][Mean, $N=2500$.]{\includegraphics[width=.48\textwidth]{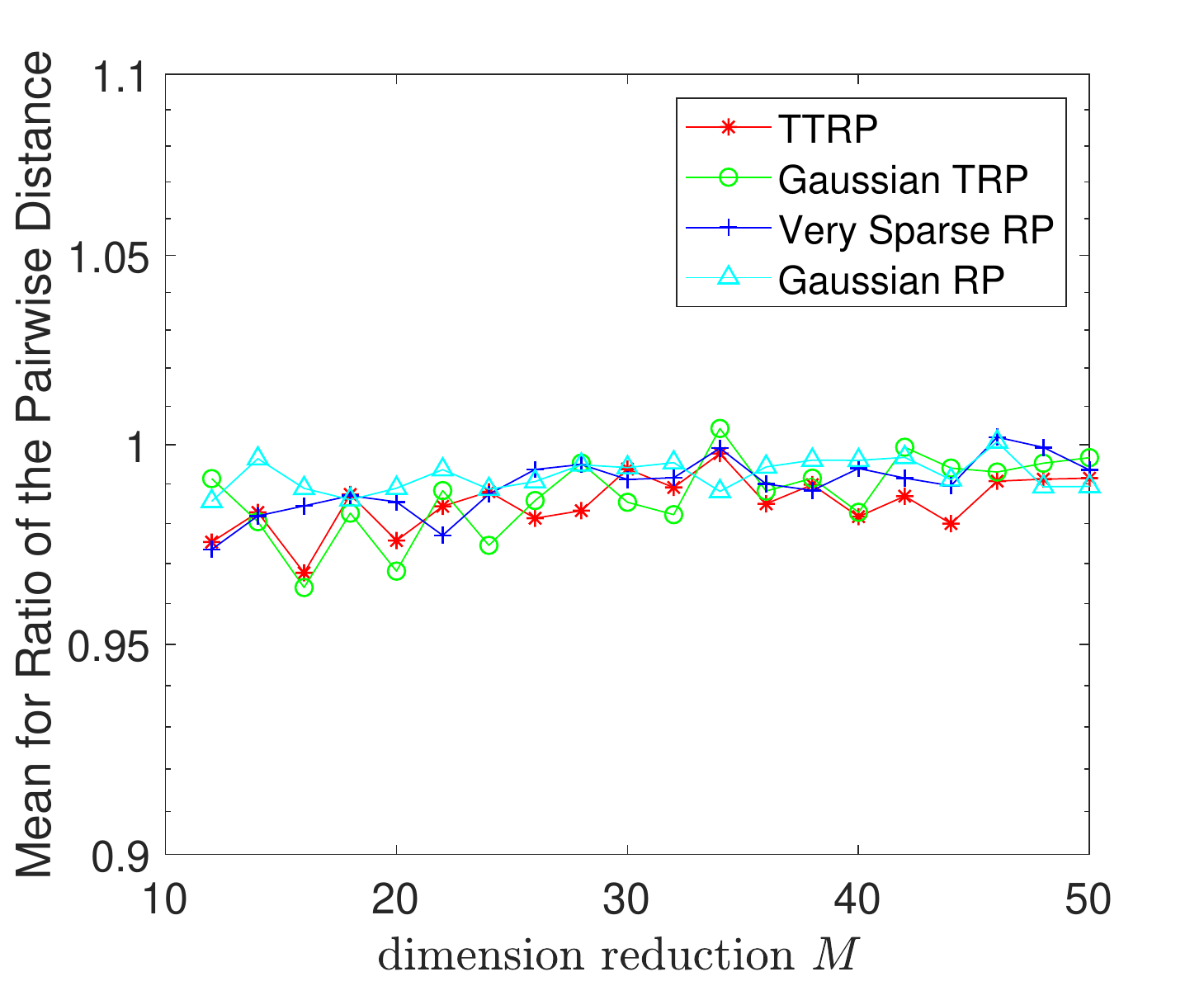}}\quad
	\subfloat[][Variance, $N=2500$.]{\includegraphics[width=.48\textwidth]{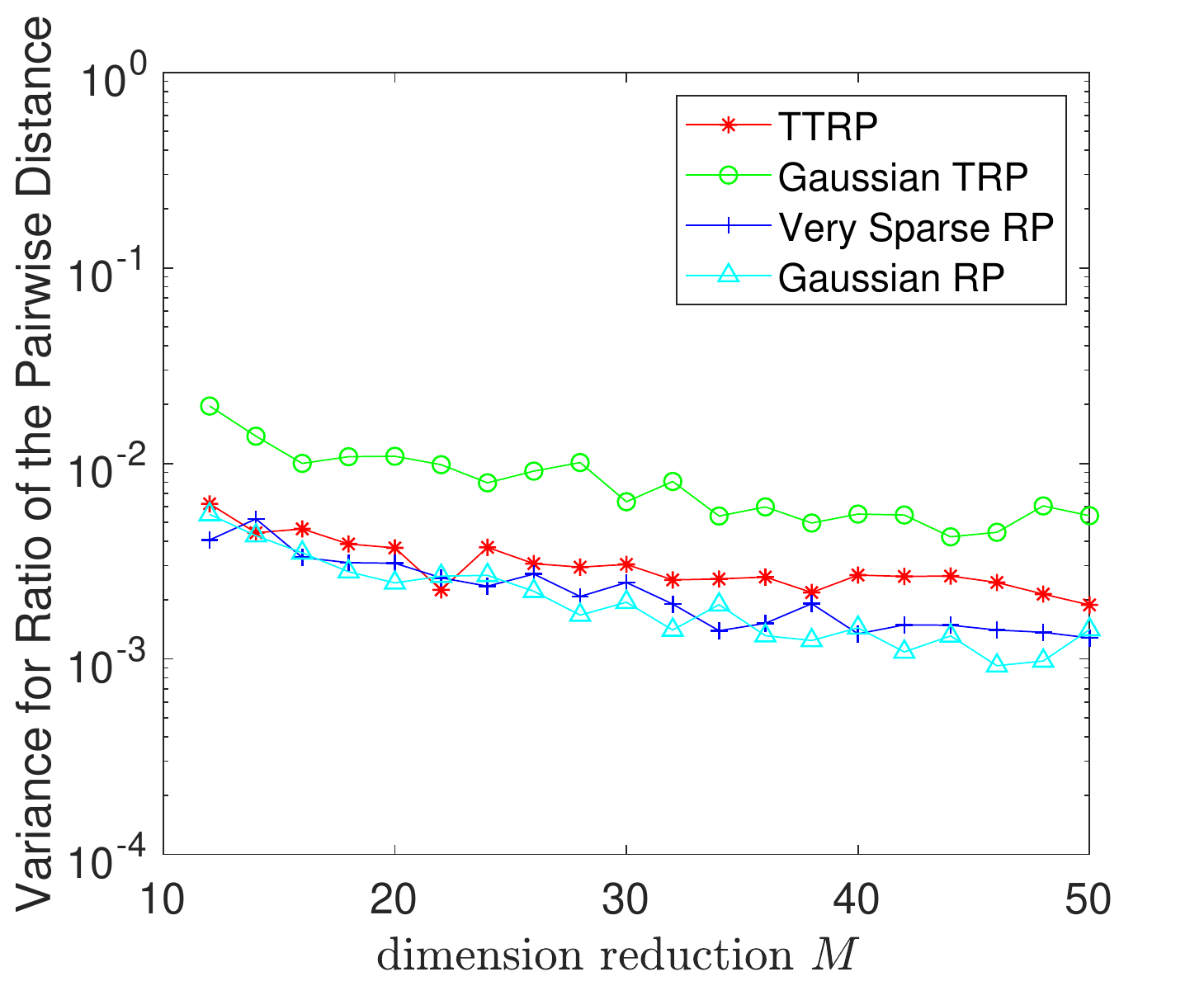}}
	\\
	\subfloat[][Mean, $N=10^4$.]{\includegraphics[width=.48\textwidth]{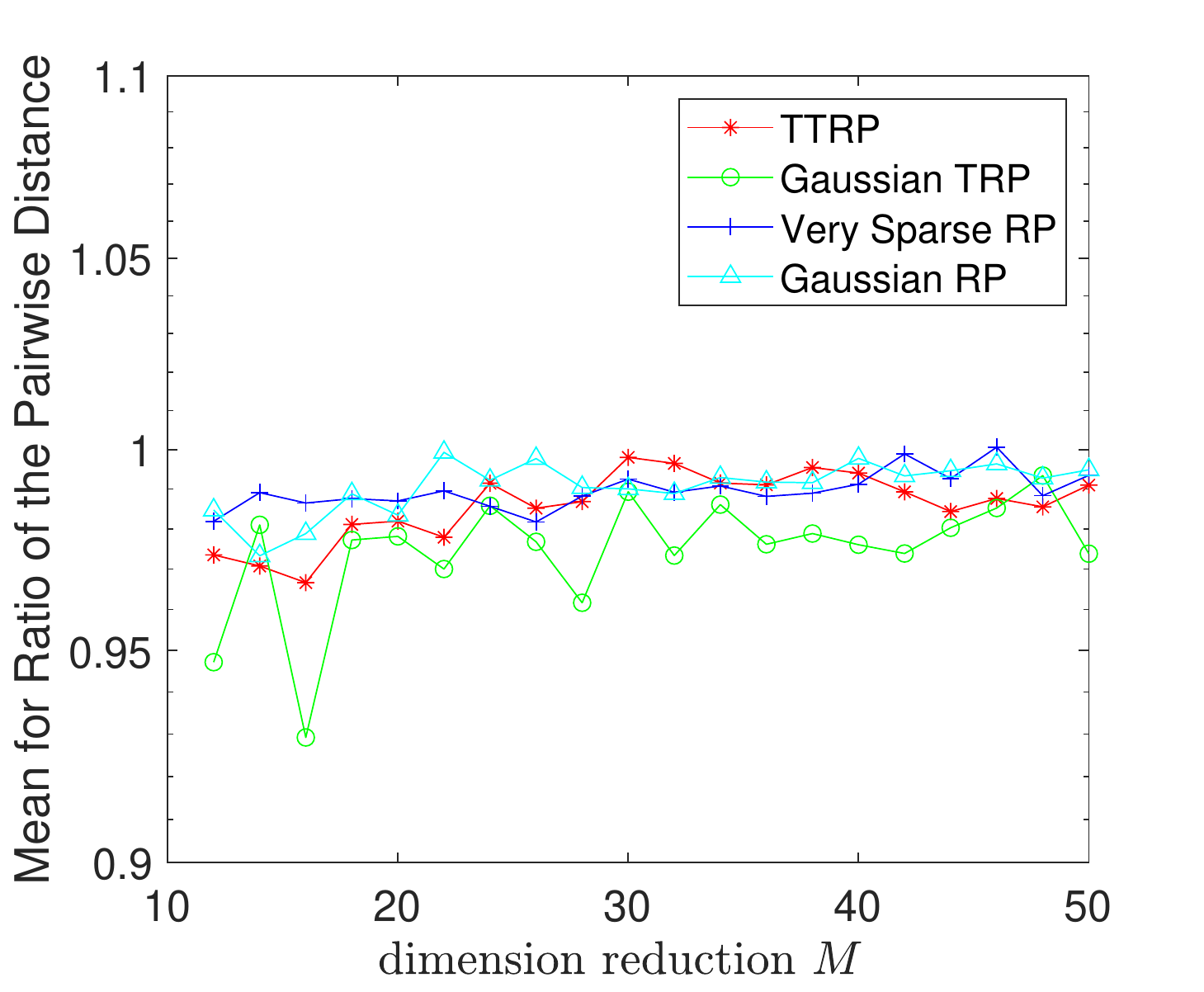}}\quad
	\subfloat[][Variance, $N=10^4$.]{\includegraphics[width=.48\textwidth]{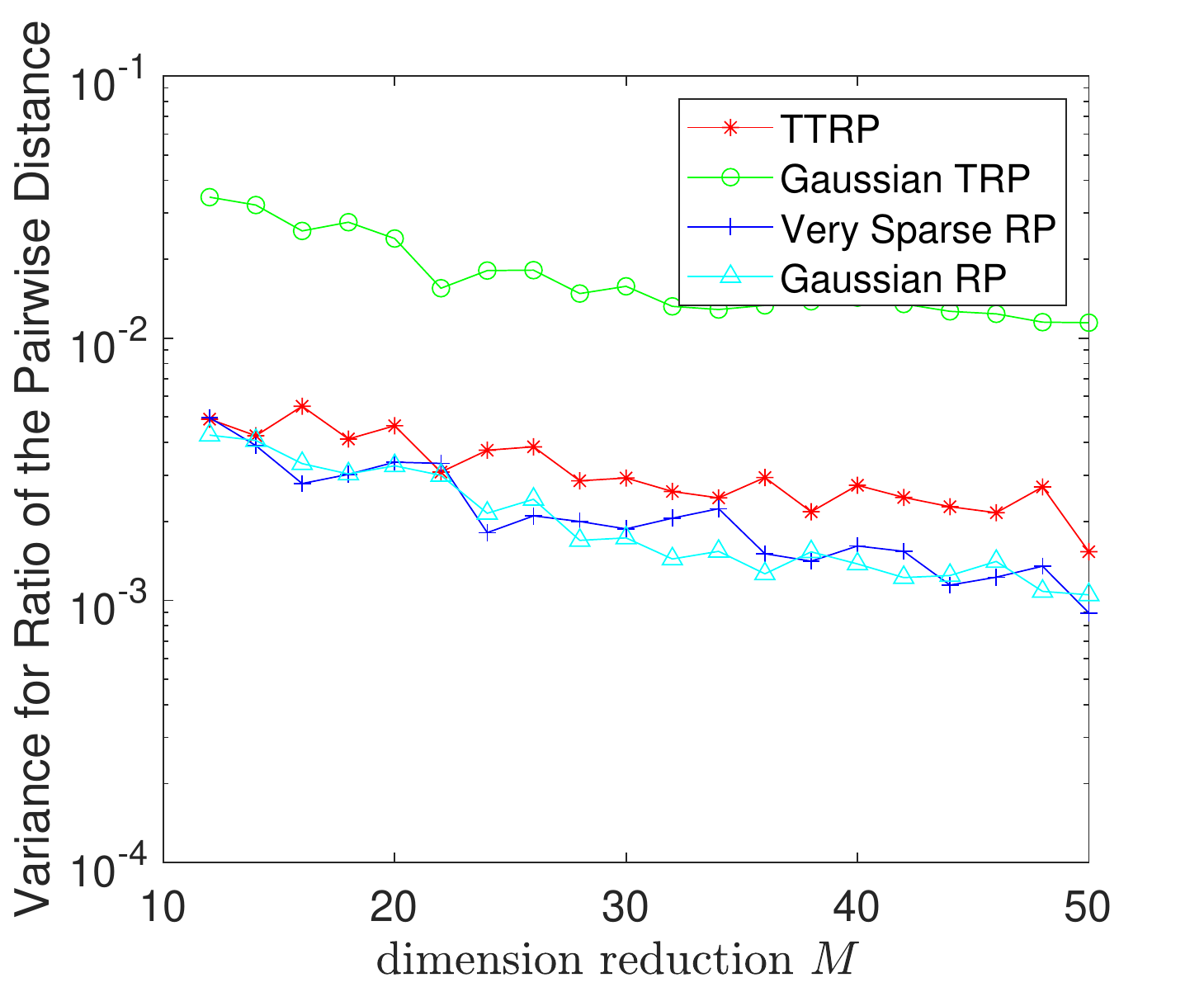}}
	\caption{Mean and variance for the ratio of the pairwise distance,  synthetic data.}
    \label{s2500}
\end{figure}
	\begin{table}
		\caption{The comparison of mean and variance for the ratio of the pairwise distance, and storage, for Gaussian RP and Very Sparse RP ($M=24$ and $N=10^4$).}
		\label{estorage1}
		\centering
		\begin{tabular}{c|c|c|c|c|c}
			\hline
			\multicolumn{3}{c|}{Gaussian RP} & \multicolumn{3}{|c}{Very Sparse RP}\\  
			\hline
			mean&variance&storage&mean&variance&storage\\
			\hline
			0.9908&0.0032&240000&0.9963&0.0025&2400\\
			\hline
		\end{tabular}
	\end{table}
	\begin{table}
		\caption{
		The comparison of mean and variance for the ratio of the pairwise distance, and storage, for Gaussian TRP and TTRP ($M=24$ and $N=10^4$).}
		\label{estorage2}
		\centering
		\begin{tabular}{c|c|c|c|c|c|c|c}
			\hline
			\multicolumn{2}{c|}{Dimensions for tensorization}&\multicolumn{3}{c|}{Gaussian TRP} & \multicolumn{3}{|c}{TTRP}\\  
			\hline
			$[m_1,\ldots,m_d]$ & $[n_1,\dots,n_d]$ & mean & variance & storage & mean & variance & storage\\
			\hline
			[6,4] & [100,100] & 0.9908 & 0.0026 & 4800 & 0.9884 & 0.0026 & 1000\\
			\hline
			[4,3,2]&[25,20,20]& 0.9747 & 0.0062 & 1560 & 0.9846& 0.0028 & 200\\
			\hline
			[3,2,2,2] & [10,10,10,10] & 0.9811 & 0.0123 & 960 & 0.9851& 0.0035 & 90\\
			\hline
		\end{tabular}
	\end{table}

Next the CPU times for projecting a data point using the four 
 methods (TTRP, Gaussian TRP, Very Sparse RP and Gaussian RP) are assessed. Here, we
set the reduced dimension $M=1000$, and test four cases with $N=10^4$, $N=10^5$, $N=2\times 10^4$ and $N=10^6$ respectively.
The dimensions of the tensorized output 
is set to $m_1=m_2=m_3=10$ (such that $M=m_1m_2m_3$), and the dimensions of the corresponding tensor representations of the 
original data points are set to: 
for $N=10^4$,
$n_1=25,\,n_2=25,\,n_3=16$;
for $N=10^5$, $n_1=50,\,n_2=50,\,n_3=40$; 
for $N=2\times 10^5$, $n_1=80,\,n_2=50,\,n_3=50$; 
for $N=10^6$, $n_1=n_2=n_3=100$.
For each case, given a data point of which elements are sampled from the standard Gaussian distribution, the simulation of projecting it to the reduced dimensional space  is repeated 100 times (different realizations for constructing the random projections), and the CPU time is defined to be the average time of these 100 simulations. 
Figure \ref{fig_time_comparison} shows the CPU times, where the results are obtained in 
MATLAB on a workstation with Intel(R) Xeon(R) Gold 6130 CPU. It is clear that the computational cost of our TTRP is much smaller than those of Gaussian TRP and Gaussian RP for different data dimension $N$. As the data dimension $N$ increases, the computational costs of Gaussian TRP and Gaussian RP grow rapidly, while the computational cost of our TTRP grows slowly.
When the data dimension is large (e.g., $N=10^6$ in Figure \ref{fig_time_comparison}),
the CPU time of TTRP becomes smaller than that of Very Sparse RP, which is consist with the results in Table \ref{storage}.

	\begin{figure}
		\centering
		\includegraphics[width=3.8in,height=2.6in]{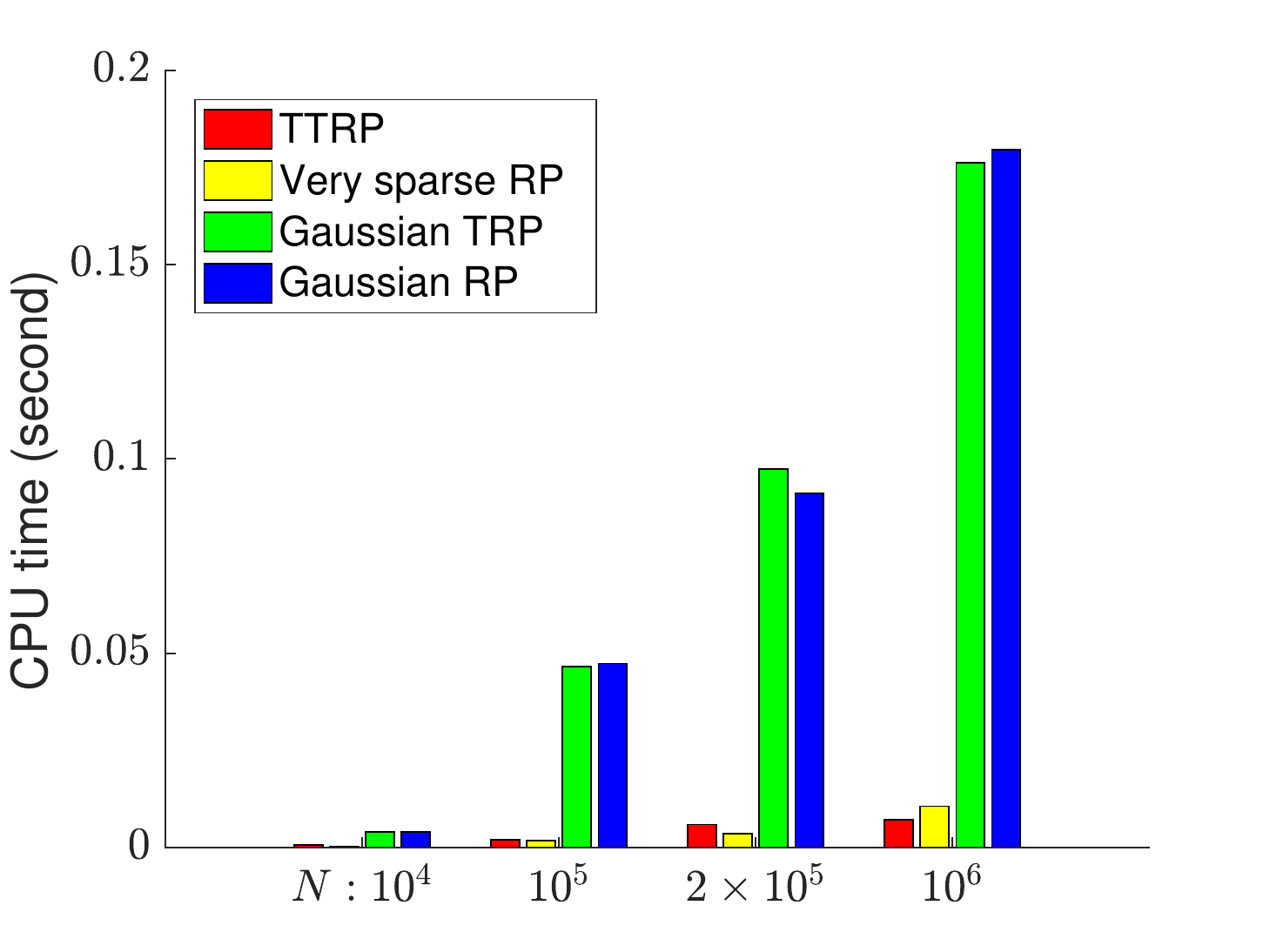}
		\caption{A comparison of CPU time for different random projections ($M=1000$).}
		\label{fig_time_comparison}
	\end{figure}
Finally, we validate the performance of our TTRP approach using the MNIST dataset \cite{lecun2010mnist}.
From MNIST, we randomly take $n_0=50$ data points, each of which is a vector with dimension $N = 784$. We consider two cases for the dimensions of tensor representations: in the first case, we set $m_1=M/2,\,m_2=2,\,n_1=196,\,n_2=4$,
and in the second case, we set  $m_1=M/2,\,m_2=2,\,m_3=1,\,n_1=49,\,n_2=4,\,n_3=4$. Figure \ref{minist} shows the properties of isometry and bounded variance of different random projections on MNIST. 
It can be seen that TTRP satisfies the isometry property with bounded variance. 
It is clear that as the reduced dimension $M$ increases, the variances of the four methods reduce,
and the variance of our TTRP is close to that of Very Sparse RP.

\begin{figure}[!ht]
	\centering
	\subfloat[][Mean for the ratio of the pairwise distance ]{\includegraphics[width=.48\textwidth]{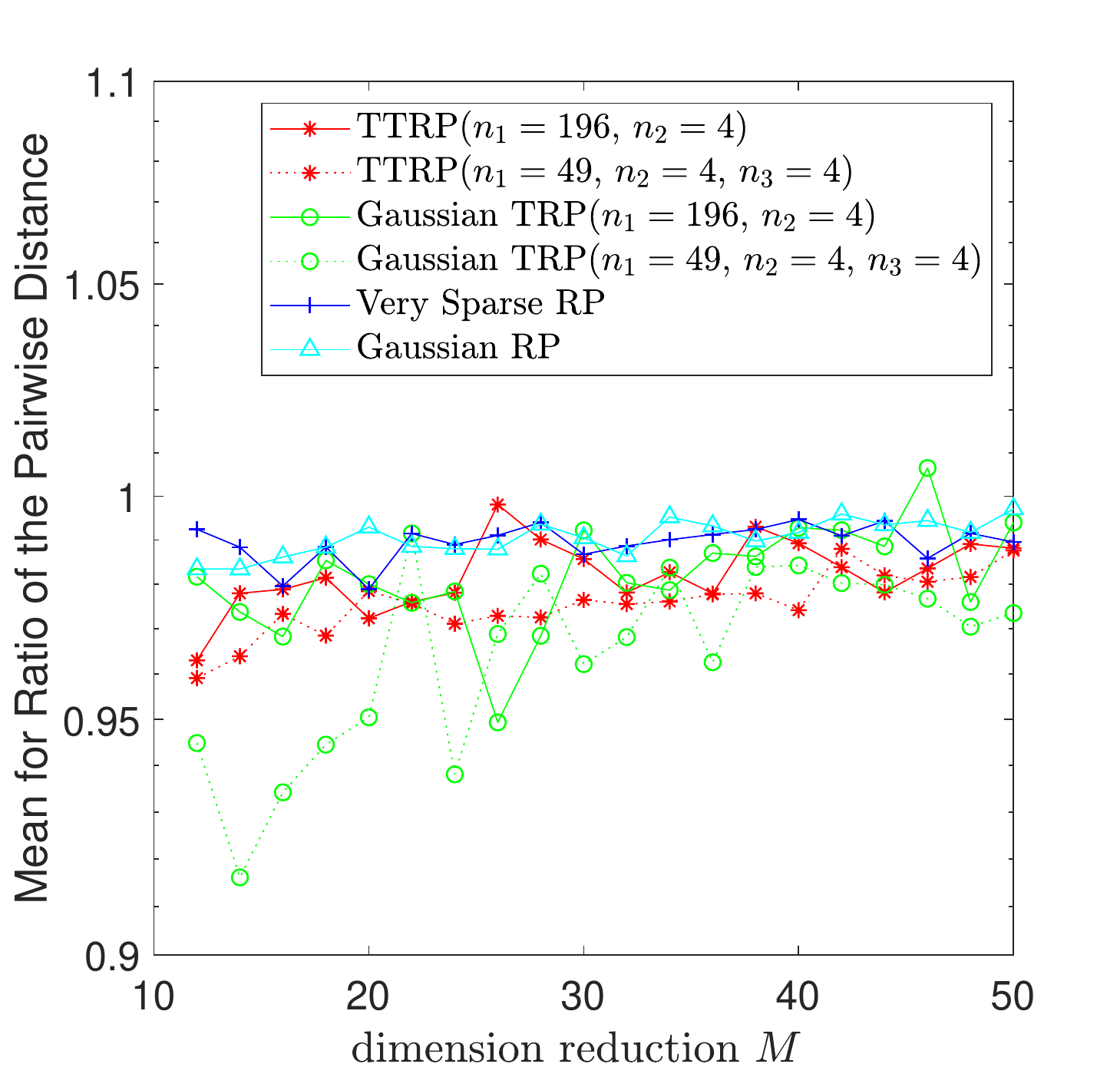}}\quad
	\subfloat[][Variance for the ratio of the pairwise distance]{\includegraphics[width=.48\textwidth]{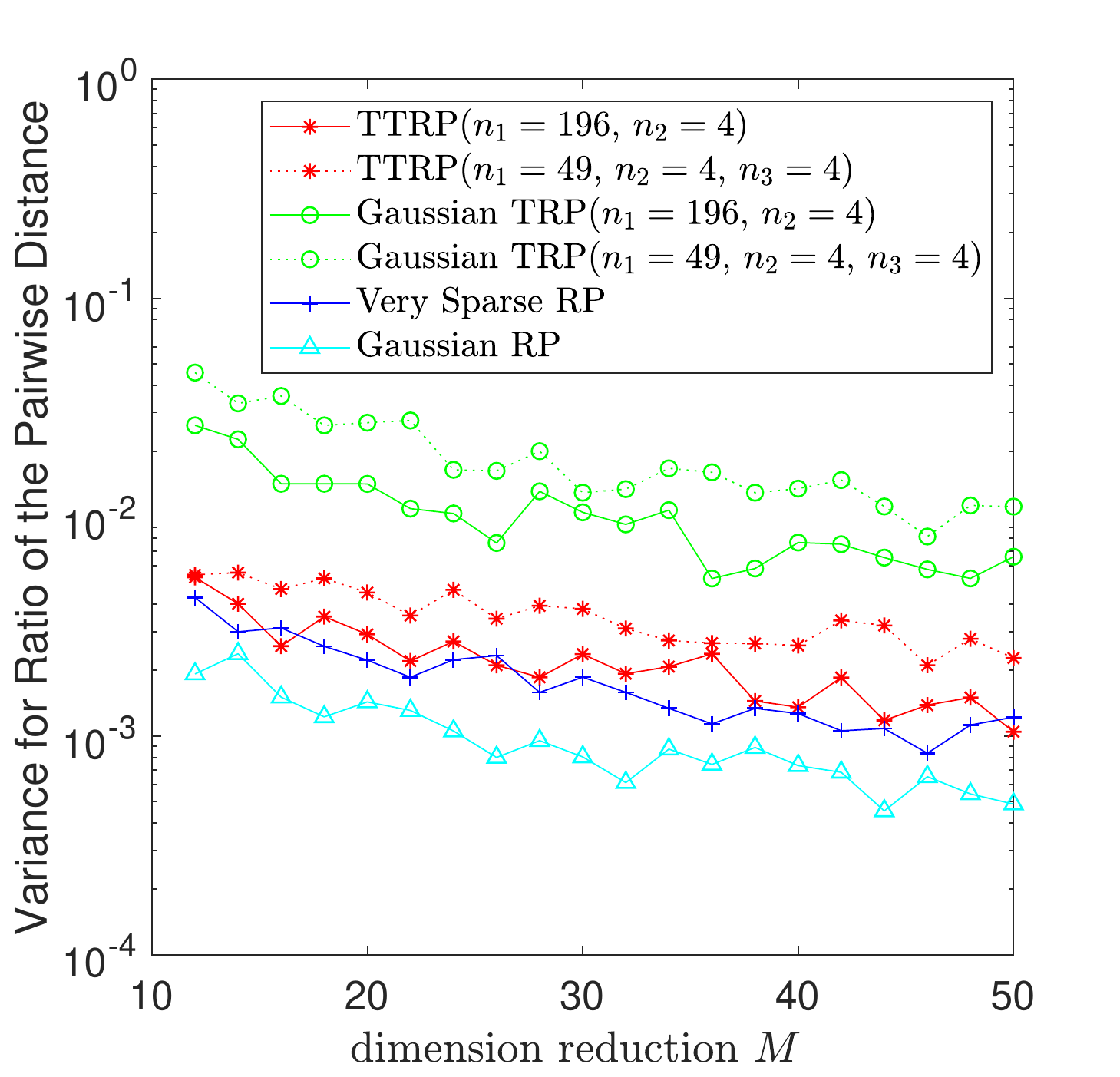}}
	\caption{Isometry and variance quality for MNIST data ($N=784$).}
    \label{minist}
\end{figure}

\section{Conclusion}\label{Conclu}

Random projection plays a fundamental role in
conducting dimension reduction for high-dimensional datasets, where pairwise distances need to be approximately preserved. With a focus on efficient tensorized computation, this paper develops a novel tensor train random projection (TTRP) method.
Based on our analysis for the bias and the variance,
TTRP is proven to be an expected isometric projection with bounded variance. From the analysis in Theorem \ref{lemma_var}, the Rademacher distribution is shown to be an optimal choice to generate the TT-cores of TTRP. For computational convenience, the TT-ranks of TTRP are set to one, while from our numerical results, we show that different TT-ranks do not lead to significant results for the mean and the variance of the ratio of the pairwise distance. 
Our detailed numerical studies show that, compared with standard projection methods, our TTRP with the default setting (TT-ranks equal one and TT-cores are generated through the Rademacher distribution), requires 
significantly smaller storage and computational costs to achieve a competitive performance.   
From numerical results, we also find that our TTRP has smaller variances than tensor train random projection methods based on Gaussian distributions.
Even though we have proven the properties of the mean and the variance of TTRP and the numerical results show that TTRP is efficient,
the upper bound in the concentration inequality \eqref{prop2_ineq} involves the dimensionality of datasets ($N$), and our future work is to  give a tight bound independent of the dimensionality of datasets for the concentration inequality.

\bmhead{Acknowledgments}
The authors thank Osman Asif Malik and Stephen Becker for helpful suggestions and
discussions.

 This work is supported by the National Natural Science Foundation of China (No. 12071291), the Science and Technology Commission of Shanghai Municipality (No. 20JC1414300) and the Natural Science Foundation of Shanghai (No. 20ZR1436200).

 \section*{Declarations}


 
 The authors declare that they have no known competing financial interests or personal relationships that could have
appeared to influence the work reported in this paper.







\begin{appendices}

\section{Example for $\mathbb{E}[\mathbf{y}^2(i)\mathbf{y}^2(j)]\neq \mathbb{E}[\mathbf{y}^2(i)]\mathbb{E}[\mathbf{y}^2(j)],\,i\neq j$.}\label{appendix}

If all TT-ranks of tensorized matrix $\mathbf{R}$ in \eqref{eq_ttfpfull} are equal to one, then $\mathbf{R}$  is represented as a Kronecker product of $d$ matrices,
$$
\mathbf{R}=\mathbf{R}_1 \otimes \mathbf{R}_2\otimes\cdots \otimes \mathbf{R}_d,
$$
where $\mathbf{R}_i \in \mathbb{R}^{m_i\times n_i}$, for $i=1,2,..,d$, whose entries are i.i.d. mean zero and variance one. We just consider $d=2, m_1=m_2=n_1=n_2=2$, then
$$
\mathbf{y}=\mathbf{R}\mathbf{x}=(\mathbf{R}_1 \otimes \mathbf{R}_2)\mathbf{x},
$$
where 
\begin{equation*}
    \mathbf{R}_1=\left[ \begin{array}{cc}
    a_1 & a_2\\
    b_1&b_2
    \end{array}
    \right],\,
    \mathbf{R}_2=\left[ \begin{array}{cc}
    c_1 & c_2\\
    d_1&d_2
    \end{array}
    \right].
\end{equation*}
Hence
\begin{equation*}
    \mathbf{y}=\left[\begin{array}{c}
         \mathbf{y}(1) \\
         \mathbf{y}(2)\\
          \mathbf{y}(3) \\
         \mathbf{y}(4)
    \end{array}
    \right]=\left[ \begin{array}{c}
    a_1c_1x_1 +a_1c_2x_2+a_2c_1x_3+a_2c_2x_4\\
    a_1d_1x_1+a_1d_2x_2+a_2d_1x_3+a_2d_2x_4\\
    b_1c_1x_1+b_1c_2x_2+b_2c_1x_3+b_2c_2x_4\\
    b_1d_1x_1+b_1d_2x_2+b_2d_1x_3+b_2d_2x_4
    \end{array}
    \right].
\end{equation*}
We compute the following,
\begin{align*}
    \text{cov}&\Big(\mathbf{y}^2(1),\mathbf{y}^2(3)\Big)\\
    =&\text{cov}\left(\left(a_1c_1x_1 +a_1c_2x_2+a_2c_1x_3+a_2c_2x_4\right)^2,\left( b_1c_1x_1+b_1c_2x_2+b_2c_1x_3+b_2c_2x_4\right)^2\right)\\
    =&\text{cov}\left(a_1^2c_1^2x_1^2+a_1^2c_2^2x_2^2+a_2^2c_1^2x_3^2+a_2^2c_2^2x_4^2,b_1^2c_1^2x_1^2+b_1^2c_2^2x_2^2+b_2^2c_1^2x_3^2+b_2^2c_2^2x_4^2\right)\\
    &+\text{cov}\left(2a^2_1c_1c_2x_1x_2+2a^2_2c_1c_2x_3x_4,2b^2_1c_1c_2x_1x_2+2b^2_2c_1c_2x_3x_4\right)\\
    =&\left(x^2_1+x^2_3\right)^2\text{var}(c^2_1)+\left(x^2_2+x^2_4\right)^2\text{var}(c^2_2)+4\left(x_1x_2+x_3x_4\right)^2\text{var}(c_1c_2)\\
    =&\left(x^2_1+x^2_3\right)^2\text{var}(c^2_1)+\left(x^2_2+x^2_4\right)^2\text{var}(c^2_2)+4\left(x_1x_2+x_3x_4\right)^2 > 0,
\end{align*}
then $\mathbb{E}\left[\mathbf{y}^2(1)\mathbf{y}^2(3)\right]\neq\mathbb{E}\left[\mathbf{y}^2(1)\right]\mathbb{E}\left[\mathbf{y}^2(3)\right]$. 
Generally, for some $i\neq j$, $\mathbb{E}[\mathbf{y}^2(i)\mathbf{y}^2(j)]\neq \mathbb{E}[\mathbf{y}^2(i)]\mathbb{E}[\mathbf{y}^2(j)]$.

\end{appendices}


\bibliography{feng}


\end{document}